\documentclass{article}

\usepackage{PRIMEarxiv}

\usepackage{amsmath,amsfonts,bm}

\usepackage{amsmath,amssymb,amsthm,mathtools}

\newtheorem{assumption}{Assumption}
\newtheorem{theorem}{Theorem}
\newtheorem{corollary}{Corollary}
\newtheorem{lemma}{Lemma}

\def\eqref#1{equation~\ref{#1}}

\def\1{\bm{1}}

\DeclareMathAlphabet{\mathsfit}{\encodingdefault}{\sfdefault}{m}{sl}
\SetMathAlphabet{\mathsfit}{bold}{\encodingdefault}{\sfdefault}{bx}{n}

\usepackage{natbib}
\usepackage[utf8]{inputenc} 
\usepackage[T1]{fontenc}    
\usepackage{hyperref}       
\usepackage{url}            
\usepackage{booktabs}       
\usepackage{amsfonts}       
\usepackage{nicefrac}       
\usepackage{microtype}      
\usepackage{lipsum}
\usepackage{fancyhdr}       
\usepackage{graphicx}       
\graphicspath{{media/}}     
\usepackage{amsmath}
\usepackage{wrapfig}
\usepackage{float}
\usepackage{caption}
\usepackage{hyperref}
\usepackage{url}
\usepackage{booktabs}
\usepackage{ulem}
\usepackage{mathrsfs}
\usepackage{graphicx} 
\usepackage{xcolor}
\usepackage{subcaption}
\usepackage{algorithm}
\usepackage{algpseudocode}
\title{Continuous-Time Trajectory Generation from Discrete Observations with Stochasticity}

\author{Ruifeng Shang \thanks{Equal contribution.}\\
Department of Statistics\\
University of Chicago\\
Chicago, IL 60637, USA \\
\texttt{rfshang@uchicago.edu} \\
\And
Shu Liu \footnotemark[1]\\
Department of Mathematics \\
Florida State University \\
Tallahassee, FL 32306, USA \\
\texttt{sliu11@fsu.edu} 
\And
Yuhua Zhu \thanks{Corresponding author.} \\
Department of Statistics and Data Science \\
University of California, Los Angeles \\
Los Angeles, CA 90095, USA \\
\texttt{yuhuazhu@ucla.edu} 
}

\begin{document}
\maketitle

\begin{abstract}
Physical systems evolve continuously in time, yet their states are typically observed only at discrete times. Generating trajectories consistent with their probability densities from such observations therefore requires capturing the continuous-time evolution rather than only learning transition mappings between consecutive observations. We propose \textbf{PhiBE-Flow}, a framework that directly estimates the probability velocity field induced by the stochastic differential equation (SDE) which governs this continuous-time distributional evolution. PhiBE-Flow learns from discrete observations using a model-free approach requiring neither known SDE coefficients nor score estimation. We establish convergence guarantees for the method, accounting for both time-discretization and finite-sample errors. We evaluate PhiBE-Flow on systems of increasing complexity, from controlled stochastic numerical systems to Navier--Stokes dynamics and real-world videos. Our results show that PhiBE-Flow accurately recovers probability flows of stochastic dynamics, preserves multiscale physical statistics, and improves video generation performance over representative baselines. The code is available at \url{https://github.com/R1fe/PhiBE-Flow}.
\end{abstract}

\keywords{Generative modeling, Continuous-time generation, Stochastic dynamical systems, Probability flows, Video generation}

\section{Introduction}
\begin{figure}[!ht]
    \centering
    \includegraphics[width=\linewidth]{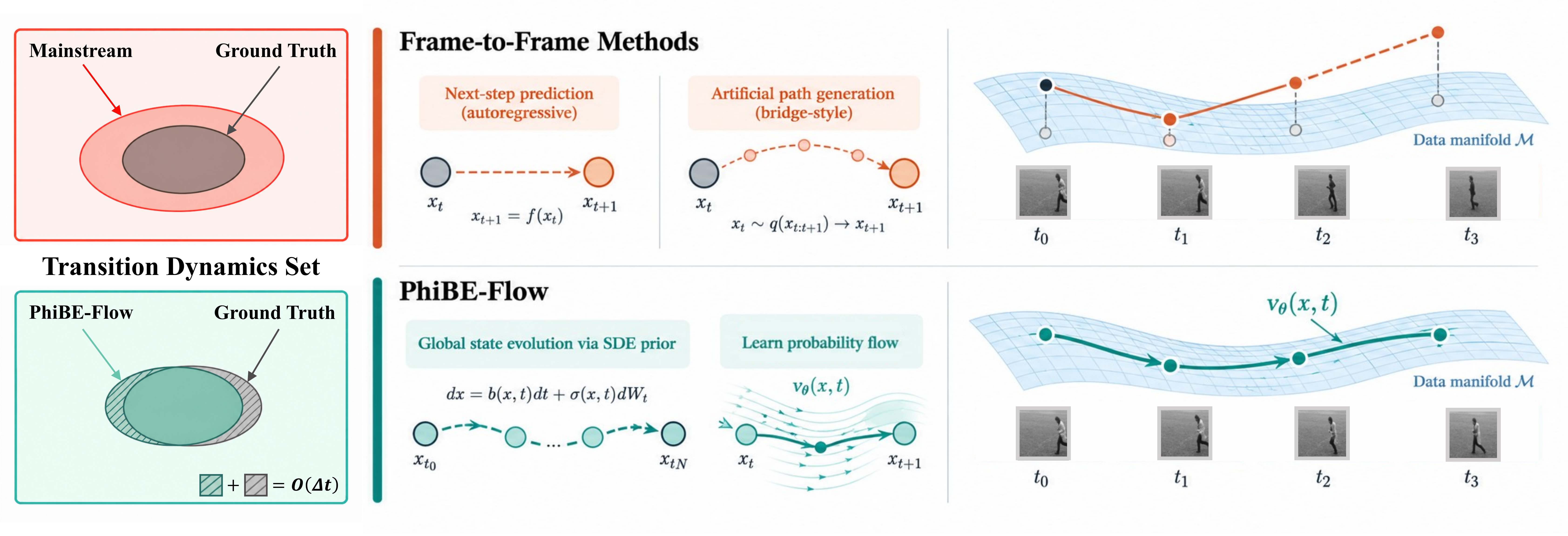}
    \caption{\textbf{Overview of PhiBE-Flow.}
        Top: frame-to-frame methods learn relations between sampled
        states.
        Bottom: PhiBE-Flow directly learns the SDE-induced probability
        velocity.}
    \label{fig:overview}
\end{figure}
Real-world physical systems are typically observed through discrete
measurements. Learning their dynamics has attracted growing interest
across applications in mechanics
\citep{greydanus2019hamiltonian,lutter2019deep},
fluid dynamics
\citep{brunton2020machine,kochkov2021machine,buaria2023forecasting}, and video modeling \citep{oprea2020review, finn2016unsupervised, lee2018stochastic}. Their underlying dynamics, however, often evolve continuously in physical time and may be stochastic or chaotic. Learning these systems therefore calls for capturing their continuous-time evolution beyond modeling transition mappings between sampled states. This is particularly important for video, where visual realism at the frame level does not necessarily imply coherent motion or consistency with the underlying physical dynamics
\citep{bansal2026videophy,gu2026phyworldbench,motamed2026generative}.

Recent advances in generative modeling have opened new avenues for learning dynamical systems from discrete observations \citep{ho2022video,lippe2023pde,price2025probabilistic}. Autoregressive models learn conditional transitions between successive states or tokens, using architectures such as recurrent neural networks \citep{ranzato2014video, vlachas2018data} and Transformers \citep{geneva2022transformers,yan2021videogpt}. Bridge-based methods instead introduce an auxiliary path between neighboring states \citep{albergo2025stochastic}, with diffusion models providing a prominent example through an artificial noising--denoising process \citep{ho2020denoising,song2020score}. Despite their success, these approaches do not explicitly model the underlying continuous-time evolution in physical time. Neural ordinary differential equations (ODEs) and stochastic differential equations (SDEs) address this limitation by learning continuous-time dynamics \citep{li2020scalable,djeumou2023physics}. For stochastic systems, however, the underlying drift and diffusion are generally unknown, and the corresponding probability velocity additionally depends on the unknown time-dependent density score. This raises a central question: Can we learn this velocity directly from discrete observations without separately estimating these quantities?

We propose \textbf{PhiBE-Flow}, a model-free framework for continuous-time
trajectory generation from discrete observations. Using PhiBE theory
\citep{zhu2024phibe}, we approximate local drift and diffusion information
from short-time increments. Combined with the Fokker--Planck equation,
these approximations yield a score-free variational objective for learning
the probability velocity field induced by the SDE governing
the system's stochastic dynamics. With this field, we generate deterministic 
trajectories whose time-dependent marginal distributions approximate those
of the underlying stochastic system.
We evaluate PhiBE-Flow on the Ornstein--Uhlenbeck (OU) process, double-pendulum
dynamics \citep{sutton1998reinforcement}, stochastic Navier--Stokes dynamics \citep{chen2024probabilistic}, and KTH videos \citep{schuldt2004recognizing}.
These experiments demonstrate its ability to recover probability flows, faithfully capture the nonlinear motion in stochastic dynamics, 
and preserve multiscale physical statistics across systems of increasing
complexity. On both the Navier--Stokes and KTH video datasets, PhiBE-Flow
outperforms representative diffusion-based and stochastic-transport
baselines.

Our main contributions are:
\begin{itemize}
    \item We introduce \textbf{PhiBE-Flow}, a model-free framework that
    learns SDE-induced probability velocity fields from discrete observations
    for continuous-time trajectory generation.

    \item We derive a score-free variational objective without assuming known
    SDE coefficients, with convergence guarantees accounting for
    time-discretization and finite-sample errors.

    \item We demonstrate the effectiveness of PhiBE-Flow on stochastic dynamical systems, showing improved generation accuracy, physical consistency, and video prediction metrics over representative baselines on the Navier--Stokes and KTH video datasets.
\end{itemize}

\section{Related Work}

\subsection{Autoregressive Generation}

Autoregressive methods factorize the conditional distribution of a future sequence and generate predictions one step at a time. Recurrent models propagate a hidden state and directly predict the next state or frame~\citep{mikolov2010recurrent,sutskever2014sequence,shi2015convolutional}. Transformer-based models use self-attention to capture longer-range dependencies and often represent dynamical systems as sequences of discrete tokens or latent variables for next-token prediction~\citep{vaswani2017attention,dai2019transformer}. These models support recursive generation over arbitrary horizons and have achieved strong performance across physical and visual systems~\citep{srivastava2015unsupervised,yan2021videogpt,geneva2022transformers}. However, they learn transitions at discrete observation times, and do not explicitly model the continuous-time dynamics underlying the observations.

\subsection{Bridge-Based Generation}

Bridge-based methods replace direct prediction with an auxiliary process that transports a source distribution to the target distribution. Diffusion models use a fixed forward process to gradually corrupt data into noise, while a learned reverse process generates the target, optionally conditioned on past observations \citep{sohl2015deep,song2020score,ho2022video,voleti2022mcvd}. 
Flow Matching
directly learns continuous-time vector fields by regressing against
velocities associated with prescribed conditional probability paths
\citep{lipman2022flow,tong2023improving}.
Stochastic interpolants construct deterministic or stochastic paths between arbitrary endpoint distributions \citep{albergo2022building,albergo2025stochastic}, including Schr\"odinger bridges \citep{debortoli2021diffusion,chen2021likelihood} and F\"ollmer processes \citep{chen2024probabilistic}. Although effective for modeling forecast uncertainty, these auxiliary dynamics need not coincide with the physical-time evolution of the observed system.

\subsection{Learning Stochastic Dynamics}

A more direct approach to dynamical learning is to model temporal evolution through SDEs, whose drift and diffusion coefficients are often unknown in data-driven settings.
Neural \citep{tzen2019neural,liu2020does,kidger2021neural} and latent SDE methods \citep{hasan2021identifying,zeng2023latent,daems2024variational} learn these components using neural networks, while related physics-informed approaches infer them through Fokker--Planck equations \citep{chen2021inverse}.
To improve training efficiency, stochastic adjoint \citep{li2020scalable} and simulation-free methods \citep{course2023amortized,bartosh2025sde} have been developed for latent SDEs.
However, explicitly learning SDE dynamics can still pose scalability challenges in high-dimensional systems.

PhiBE offers an alternative perspective by incorporating local transition information into a continuous-time PDE formulation without first fitting a complete SDE model.
Building on this perspective, we retain the assumption of underlying SDE dynamics but seek a probability flow with the same time-marginal densities, without explicitly fitting drift and diffusion or numerically integrating SDEs during training.
This motivates our model-free, score-free variational formulation for learning continuous-time stochastic systems from discrete observations.

\section{Method}
\subsection{Problem Setting}

Consider a stochastic dynamical system governed by the It\^o SDE
\begin{equation}
    dX_t = b(X_t)\,dt + \sigma(X_t)\,dB_t,
    \qquad
    X_0 \sim \mu_0,
    \label{eq:sde}
\end{equation}
where $X_t\in\mathbb R^d$, while the drift field
$b:\mathbb R^d\to\mathbb R^d$ and the diffusion loading
$\sigma:\mathbb R^d\to\mathbb R^{d\times r}$ are unknown. We write
$\Sigma(x):=\sigma(x)\sigma(x)^\top$ for the diffusion covariance.

We assume that the continuous trajectory is not directly accessible. 

Instead, we observe $N$ independent trajectories at discrete time
$0=t_0<t_1<\cdots<t_{N_t}\leq T$,
\begin{equation}
    \mathcal{D}
    =
    \left\{
        x_{t_i}^{(k)}
    \right\}_{k=1,\ldots,N;\;i=0,\ldots,N_t}.
    \label{eq:dataset}
\end{equation}
Let $\mu_t=\mathrm{Law}(X_t)$ denote the marginal distribution of the
stochastic process at physical time $t$. 
Our goal is to learn the evolution of $\mu_t$ from the discrete
observations without separately identifying the complete drift and
diffusion functions. 
In particular, we seek a deterministic vector field
\begin{equation}
    v_\theta:\mathbb R^d\times[0,T]\rightarrow\mathbb{R}^d,
\end{equation}
with parameter $\theta\in\mathbb{R}^m$. Then we can sample points at intermediate times by evolving the ODE
\begin{equation}
    \frac{dX_t}{dt}
    =
    v_\theta(X_t,t),
    \qquad
    X_0\sim\mu_0.
    \label{eq:learned_flow}
\end{equation}
This deterministic dynamics is known as the \textit{probability flow} associated with the marginal evolution of the observed stochastic system.

\subsection{Probability-Flow Formulation}

Let $\rho(x,t)$ denote the probability density of $X_t$. 
The marginal evolution of the SDE in \eqref{eq:sde} is governed by
the Fokker--Planck equation
\begin{equation}
    \partial_t \rho(x,t)
    =
    -\nabla_x \cdot \bigl(b(x)\rho(x,t)\bigr)
    +
    \frac12\sum_{i,j=1}^d
    \partial_{x_i}\partial_{x_j}
    \bigl(\Sigma_{ij}(x)\rho(x,t)\bigr).
    \label{eq:fokker_planck}
\end{equation}
Using
$\nabla_x\rho(x,t)=\rho(x,t)\nabla_x\log\rho(x,t)$,
\eqref{eq:fokker_planck} can be written as the
continuity equation
\begin{equation}
    \partial_t\rho(x,t)
    +
    \nabla_x\cdot
    \bigl(\rho(x,t)v(x,t)\bigr)
    =
    0,
    \label{eq:continuity}
\end{equation}
where
\begin{equation}
    v(x,t)
    =
    b(x)
    -
    \frac12\Sigma(x)\nabla_x\log\rho(x,t)
    -
    \frac12\nabla_x\cdot\Sigma(x).
    \label{eq:probability_velocity}
\end{equation}
Thus, the marginal evolution of the stochastic system can be represented
by a deterministic probability velocity $v$.
This motivates learning $v_\theta$ by minimizing its population
discrepancy from $v$:
\begin{equation}
    \min_{\theta\in\mathbb{R}^m}
    \int_0^T
    \mathbb{E}_{X_t\sim\rho(\cdot,t)}
    \left[
        \left\|
            v_\theta(X_t,t)-v(X_t,t)
        \right\|^2
    \right]dt.
    \label{eq:population_velocity_objective}
\end{equation}

Directly evaluating ~\eqref{eq:population_velocity_objective}
is infeasible because the density score
$\nabla_x\log\rho(x,t)$ is unknown.
Expanding the squared objective, discarding terms independent of
$\theta$, and applying integration by parts to the remaining
score-dependent term, we obtain
\begin{equation}
    \min_{\theta\in\mathbb{R}^m}
    \int_0^T
    \mathbb{E}_{X_t\sim\rho(\cdot,t)}
    \left[
        \|v_\theta(X_t,t)\|^2
        -
        2v_\theta(X_t,t)\cdot b(X_t)
        -
        \Sigma(X_t):\nabla_xv_\theta(X_t,t)
    \right]dt.
    \label{eq:score_free_objective}
\end{equation}
While the density score is eliminated,
the objective still depends on the unknown drift $b$ and
diffusion covariance $\Sigma$, and
therefore cannot yet be evaluated from the discrete observations alone.

\subsection{PhiBE-Flow Objective}
\label{meth}
The score-free objective in \eqref{eq:score_free_objective} still
requires the unknown drift $b$ and diffusion covariance $\Sigma$.
To make the objective accessible from discrete observations, we leverage
the local approximations provided by the PhiBE. For a sufficiently small sampling interval
$\Delta t$, the drift and diffusion terms are approximated from
short-time transitions by
\begin{equation}
    \widehat{b}_{\Delta t}(x)
    =
    \frac{1}{\Delta t}
    \mathbb{E}
    \left[
        X_{t+\Delta t}-X_t
        \mid X_t=x
    \right],
    \label{eq:phibe_drift}
\end{equation}
and
\begin{equation}
    \widehat{\Sigma}_{\Delta t}(x)
    =
    \frac{1}{\Delta t}
    \mathbb{E}
    \left[
        (X_{t+\Delta t}-X_t)(X_{t+\Delta t}-X_t)^\top
        \mid X_t=x
    \right].
    \label{eq:phibe_diffusion}
\end{equation}

Substituting these approximations into
\eqref{eq:score_free_objective} yields the PhiBE-Flow population
objective
\begin{equation}
\begin{aligned}
    \mathcal{J}_{\Delta t}(\theta)
    =
    \int_0^T
    \mathbb{E}_{X_t\sim\rho(\cdot,t)}
    \Big[
        \|v_\theta(X_t,t)\|^2
        -2v_\theta(X_t,t)\cdot
        \widehat{b}_{\Delta t}(X_t) 
        -\widehat{\Sigma}_{\Delta t}(X_t):\nabla_xv_\theta(X_t,t)
    \Big]dt .
\end{aligned}
\label{eq:phibe_population_objective}
\end{equation}

In practice, the conditional expectations are replaced by the observed
short-time increments. For each transition, define
\begin{equation}
    \Delta x_i^{(k)}
    =
    x_{t_{i+1}}^{(k)}-x_{t_i}^{(k)},
    \qquad
    \Delta t_i=t_{i+1}-t_i .
    \label{eq:increments}
\end{equation}
The resulting empirical PhiBE-Flow objective is
\begin{equation}
\begin{aligned}
    \widehat{\mathcal{J}}(\theta)
    =
    \frac{1}{NN_t}
    \sum_{i=0}^{N_t-1}
    \sum_{k=1}^{N}
    \Bigg[
        \|v_\theta(x_{t_i}^{(k)},t_i)\|^2 
        -2v_\theta(x_{t_i}^{(k)},t_i)
        \cdot
        \frac{\Delta x_i^{(k)}}{\Delta t_i} 
        -
        \frac{1}{\Delta t_i}
        \Delta x_i^{(k)\top}
        \nabla_xv_\theta(x_{t_i}^{(k)},t_i)
        \Delta x_i^{(k)}
    \Bigg].
\end{aligned}
\label{eq:empirical_phibe_objective}
\end{equation}

After solving this optimization problem, given a new initial sample \(X_0\sim\mu_0\), we can generate
future samples by solving
\[
\dot X_t=v_\theta(X_t,t),
\qquad
X_t=X_0+\int_0^t v_\theta(X_s,s)\,ds.
\]

The covariance-weighted Jacobian term can be costly in high dimensions.
Some systems admit constant isotropic diffusion, and this term simplifies to $\sigma^2\nabla_x\cdot v_\theta$.
See Appendix~\ref{app:constant_scalar} for details.

\section{Theoretical Guarantees}
\label{sec:theory}

We establish convergence guarantees for PhiBE-Flow by separating two
sources of error: the population approximation error induced by discrete
observations and the finite-sample error arising from finitely many
transitions.

The finite-sample error is for fixed linear-basis models. Population bounds are evaluated on a compact set
$\Omega\subset\mathbb R^d$; statistical bounds use the observation-time
sampling measure $\mu_L$ in Appendix~\ref{app:finite_sample_setting}.
\subsection{Population Approximation}
\label{sec:population_approximation}

We first state the first-order consistency of the PhiBE coefficients.

\begin{lemma}[PhiBE coefficient approximation]
\label{lem:phibe_consistency}
Suppose the population conditions in Appendix~\ref{app:regularity} hold.
 Let $\widehat b$ and
$\widehat\Sigma$ denote the first-order PhiBE estimates of the drift
$b(x)$ and the diffusion covariance $\Sigma(x)$, respectively.
Then there exist constants $L_b,L_\Sigma,L_{\mathrm{div}}>0$, independent of $\Delta t$,
such that
\begin{equation}
    \|b-\widehat b\|_{\infty}
    \le L_b\Delta t,
\end{equation}
\begin{equation}
    \|\Sigma-\widehat\Sigma\|_{\infty}
    \le L_\Sigma\Delta t+o(\Delta t),
\end{equation}
and
\begin{equation}
    \|\nabla\cdot(\Sigma-\widehat\Sigma)\|_{\infty}
    \le L_{\mathrm{div}}\Delta t+o(\Delta t).
\end{equation}
\end{lemma}

The corresponding error in the induced probability velocity is
characterized as follows.

\begin{theorem}[General population approximation]
\label{thm:general_population}
Suppose the population conditions in
Appendix~\ref{app:regularity} hold, and let
$\widehat b$ and $\widehat\Sigma$ be the first-order PhiBE
coefficients in Lemma~\ref{lem:phibe_consistency}. 
Then there exist constants $L_v,K>0$, independent of $\Delta t$, such that
\begin{equation}
    \|v(x,t)-\widehat v(x,t)\|
    \le L_v\Delta t+o(\Delta t),
\end{equation}
and
\begin{equation}
    \|x_t-\widehat x_t\|
    \le
    \frac{e^{Kt}-1}{K}L_v\Delta t+o(\Delta t).
\end{equation}
\end{theorem}

The proof is provided in Appendices~\ref{app:phibe_consistency} and~\ref{app:proof_general_population}.
For the OU process, the approximate probability
velocity and the associated error bounds admit explicit expressions;
see Appendix~\ref{app:proof_ou_population}.

\subsection{Finite-Sample Convergence}
\label{sec:finite_sample}

We give a linear-basis proof for
$\mathcal V_p=\{v_\theta=\sum_{\ell=1}^p\theta_\ell\psi_\ell:
\theta\in\mathbb R^p\}$, with basis fields fixed independently of
the fitting sample. The basis fields may be nonlinear in $(x,t)$;
only their coefficients are fitted. Here $N$ counts independent
trajectories, while transitions within each trajectory may be
dependent. The sampling measure $\mu_L$ and estimator $v_N$ are
defined in Appendix~\ref{app:finite_sample_setting}.

\begin{theorem}[Finite-sample convergence for linear-basis models]
\label{thm:general_finite_sample}
Under Assumption~\ref{ass:finite_sample_class}, including
$\widehat v\in\mathcal V_p$, there exist constants
$c_0,C_{\mathcal V}>0$, independent of $N,\Delta t,\delta$, such that,
for $\Delta t\in(0,1]$ and $\delta\in(0,1)$, if
\[
N\ge c_0p\log(2p/\delta),
\]
then, with probability at least $1-\delta$,
\begin{equation}
\|v_N-\widehat v\|_{L^2(\mu_L)}
\le C_{\mathcal V}
\sqrt{\frac{p\log(2p/\delta)}{N\Delta t}}.
\end{equation}
\end{theorem}

The proof is in Appendix~\ref{app:finite_sample_setting}.
Appendix~\ref{app:proof_ou_finite_sample} gives a separate affine OU
bound relative to the population minimizer within the affine class.

\paragraph{Remark.}
Many neural networks reduce to a linear-basis model by removing
the output activation, freezing the hidden features, and training
only the final affine layer. The theorem applies under its assumptions
when the feature map is fixed independently of the fitting trajectories.
If $\widehat v\notin\mathcal V_p$, an approximation term is also
required; see Appendix~\ref{app:finite_sample_setting}.

\subsection{Overall Convergence}
\label{sec:overall_convergence}

For a compact evaluation set $\Omega$, write
\[
\|f\|_{L^2(\mu_L;\Omega)}^2
:=\frac1L\sum_{j=0}^{L-1}
\int_\Omega\|f(x,t_j)\|^2\rho_{t_j}(x)\,dx.
\]
The statistical and population errors satisfy
\begin{equation}
\label{eq:overall_error_decomposition}
\|v_N-v\|_{L^2(\mu_L;\Omega)}
\le
\underbrace{\|v_N-\widehat v\|_{L^2(\mu_L)}}
_{\text{finite-sample error}}
+\underbrace{\|\widehat v-v\|_{L^2(\mu_L;\Omega)}}
_{\text{PhiBE approximation error}}.
\end{equation}

\begin{corollary}[Overall convergence of PhiBE-Flow]
\label{cor:overall_convergence}
Under the assumptions of Theorems~\ref{thm:general_population}
and~\ref{thm:general_finite_sample}, including
$N\ge c_0p\log(2p/\delta)$, with probability at least $1-\delta$,
\[
\|v_N-v\|_{L^2(\mu_L;\Omega)}
\le L_v\Delta t+C_{\mathcal V}
\sqrt{\frac{p\log(2p/\delta)}{N\Delta t}}+o(\Delta t).
\]
\end{corollary}

The proof is in Appendix~\ref{app:proof_overall_convergence}.
This bound concerns velocity error at the observation times;
it does not by itself give a trajectory bound for the learned field.

\section{Experiments}

We evaluate PhiBE-Flow on four problems of increasing complexity: the Ornstein--Uhlenbeck process, the double pendulum system, stochastic Navier--Stokes dynamics, and KTH video generation. For the Navier--Stokes and KTH benchmarks, we compare against MCVD~\citep{voleti2022mcvd} and PFI~\citep{chen2024probabilistic}. Additional details are provided in Appendix~\ref{app:B}.

\subsection{Learning from Stochastic Dynamical Systems}

\paragraph{Ornstein--Uhlenbeck Process}
\label{sec:ou}

Consider a two-dimensional OU process with
constant diffusion,
\begin{equation}
    dX_t=-\Gamma(X_t-\mu)\,dt+\sigma\,dB_t,
    \qquad
    X_0\sim\mathcal{N}(\mu_0,\Sigma_0),
    \label{eq:ou}
\end{equation}
where $\Gamma\in\mathbb{R}^{2\times 2}, \; \sigma\in\mathbb{R}^{2\times 2}$, $\mu\in\mathbb{R}^2.$ The OU process is particularly useful for validating PhiBE-Flow
because its marginal distributions remain Gaussian and its
probability velocity admits a closed-form affine expression.
We generate trajectories from \eqref{eq:ou} and use pairs of
consecutive observations at discrete time points to learn the
spatiotemporal probability velocity field $v_\theta(x,t)$.

\begin{figure}[!h]
    \centering
    \begin{subfigure}[t]{0.186\linewidth}
        \centering
        \includegraphics[
            width=\linewidth,
            trim={0 0 0 0},
            clip
        ]{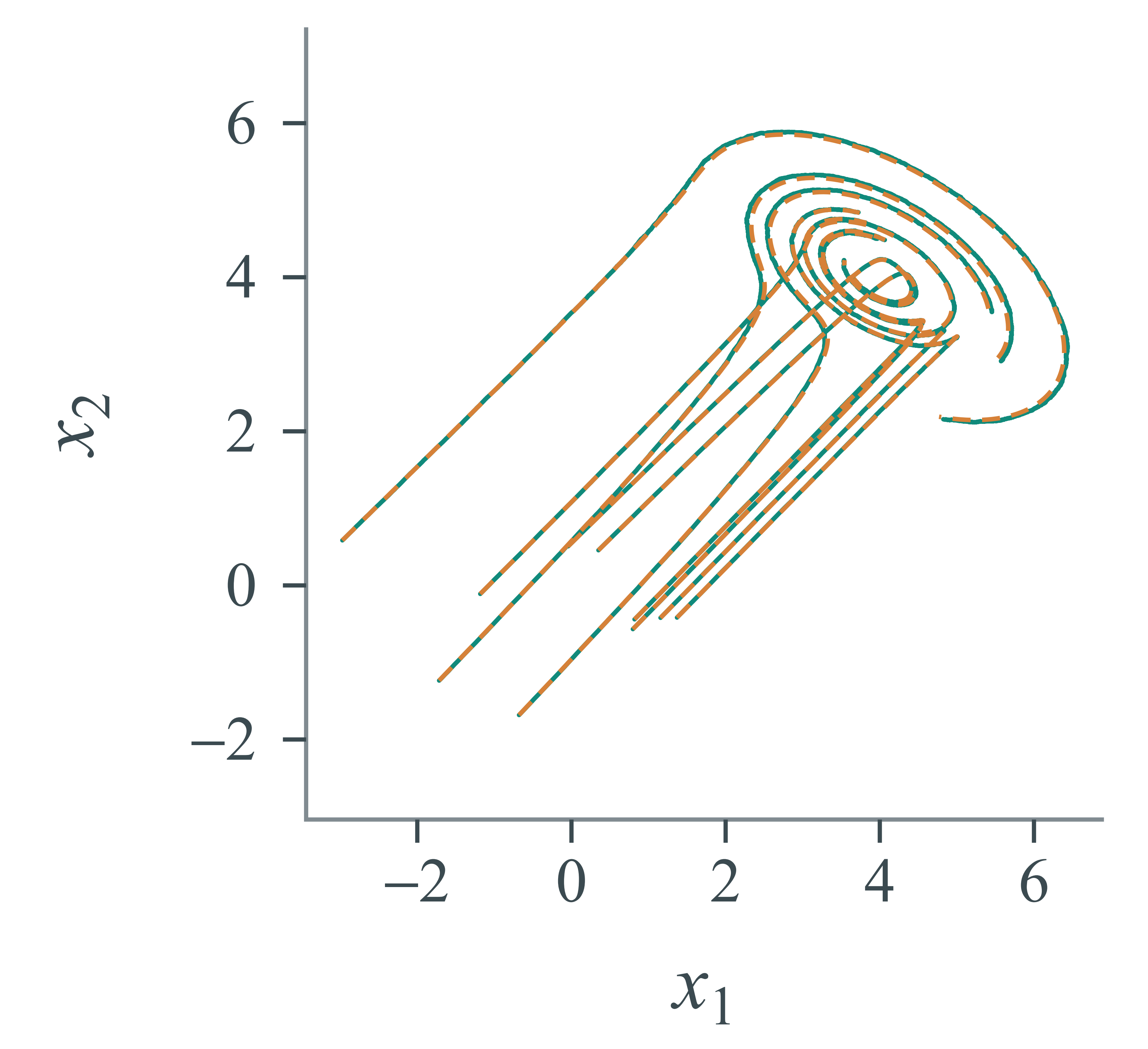}
        \caption{Probability-flow (Non-reversible).}
        \label{fig:ou_flow}
    \end{subfigure}\hfill
    \begin{subfigure}[t]{0.186\linewidth}
        \centering
        \includegraphics[
            width=\linewidth,
            trim={0 0 0 0},
            clip
        ]{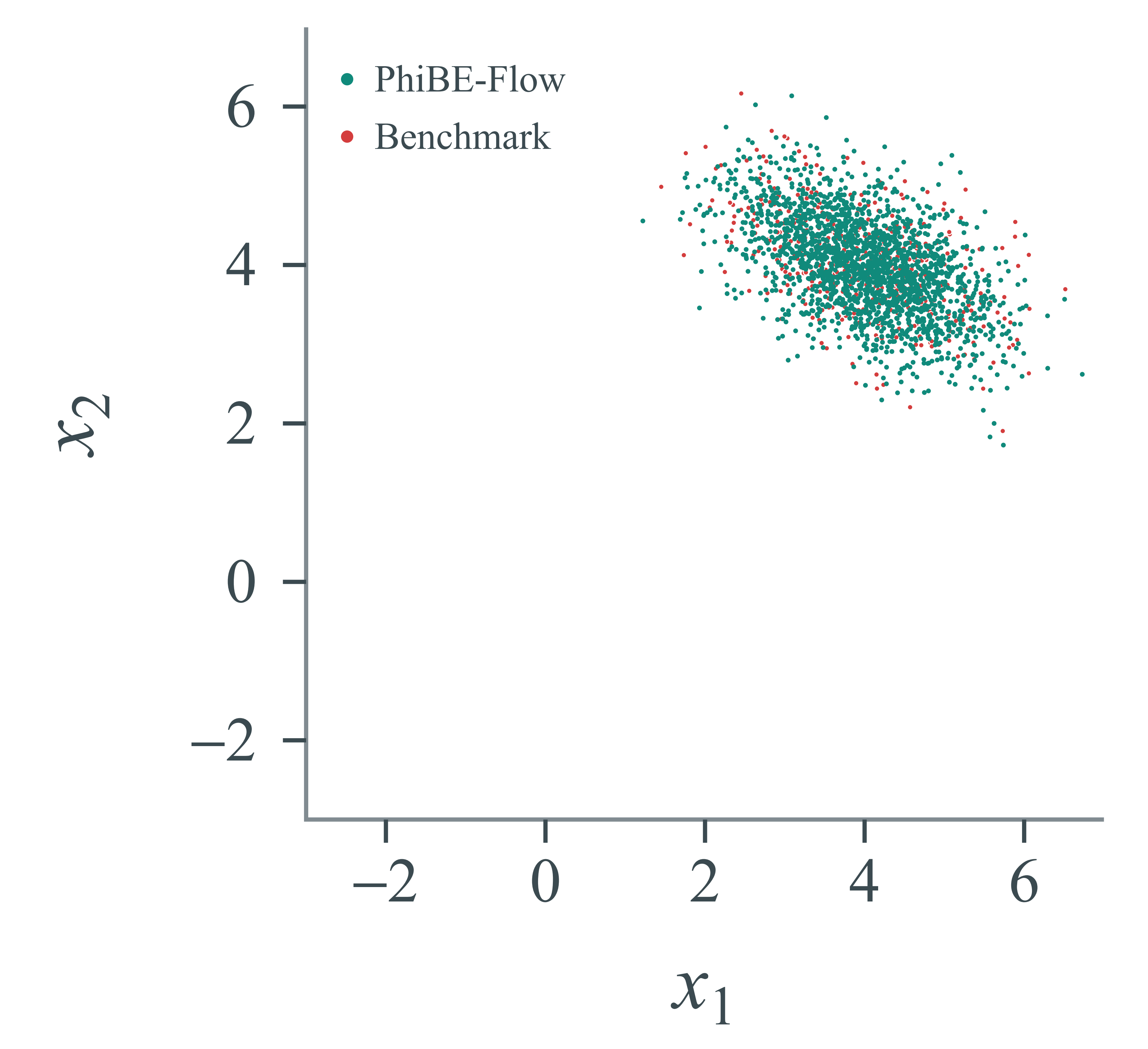}
        \caption{Samples $t=15$. (Non-reversible)}
        \label{fig:ou_flow}
    \end{subfigure}\hfill
    \begin{subfigure}[t]{0.186\linewidth}
        \centering
        \includegraphics[
            width=\linewidth
        ]{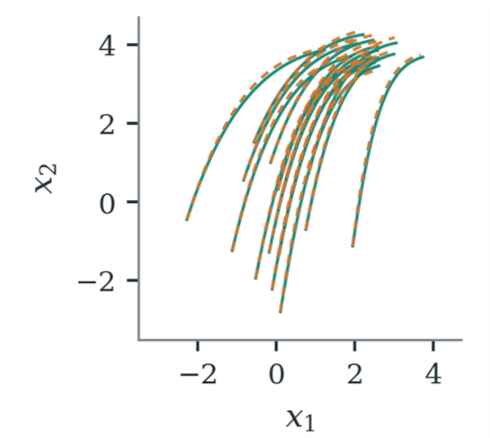}
        \caption{Probability-flow (Reversible).}
        \label{fig:ou_flow}
    \end{subfigure}
    \hfill
    \begin{subfigure}[t]{0.186\linewidth}
        \centering
        \includegraphics[
            width=\linewidth
        ]{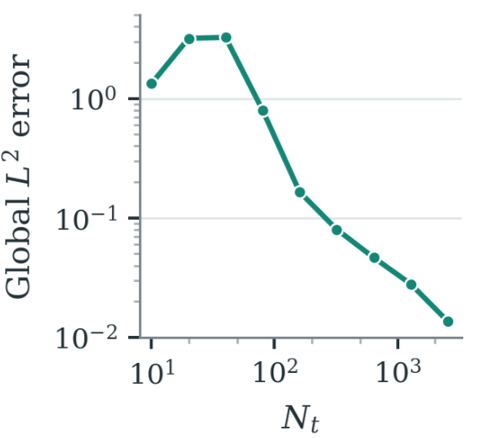}
        \caption{Error versus $N_t$. (Reversible)}
        \label{fig:ou_nt}
    \end{subfigure}
    \hfill
    \begin{subfigure}[t]{0.186\linewidth}
        \centering
        \includegraphics[
            width=\linewidth
        ]{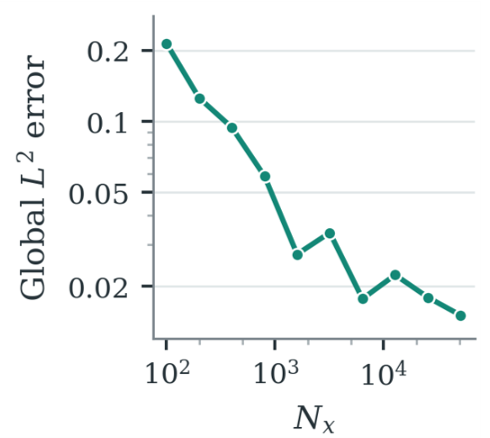}
        \caption{Error versus $N$. (Reversible)}
        \label{fig:ou_nx}
    \end{subfigure}
    \hfill
    \caption{
        \textbf{PhiBE-Flow on reversible and non-reversible OU processes.}
        (a) PhiBE-Flow (green) and analytical (orange) probability-flow trajectories for the non-reversible case.
        (b) PhiBE-Flow (blue) and OU-process (red) samples at $t=15$.
        (c) PhiBE-Flow (green) and analytical (orange) trajectories for the reversible case.
        (d) Global $L^2$ velocity error versus temporal observations $N_t$.
        (e) Global $L^2$ velocity error versus observed trajectories $N$.}
    \label{fig:ou}
    \vspace{-0.6em}
\end{figure}

We first evaluate PhiBE-Flow on a non-reversible OU process whose drift and diffusion matrices, $\Gamma$ and $\sigma$, do not commute. At each observation time $t_i$, we use the affine parameterization
$v_\theta(x,t_i)=A_i x+b_i$, which exactly represents the analytical probability velocity. Despite convergence of the marginal distribution to a stationary Gaussian (Figure~\ref{fig:ou}(b)), the process retains a nonzero steady-state probability current due to broken detailed balance
\citep{godreche2019characterising,sekizawa2024decomposing}.
PhiBE-Flow accurately reproduces these circulating probability-flow trajectories (Figure~\ref{fig:ou}(a)).

We next consider a reversible OU process using a multilayer perceptron
(MLP) parameterization. PhiBE-Flow closely recovers the analytical probability velocity and trajectories (Figure~\ref{fig:ou}(c)). Figures~\ref{fig:ou}(d) and~\ref{fig:ou}(e) further show global $L^2$ velocity errors decaying approximately as $O(N_t^{-1})$ and $O(N^{-1/2})$, respectively. While the finite-sample bound does not directly cover the jointly trained MLP, these trends are consistent with the population time-discretization and affine finite-sample scalings in Corollaries~\ref{cor:ou_population} and~\ref{cor:ou_finite_sample}.

We refer readers to Appendix~\ref{app:ou} for further details on data generation and the implementation of the algorithm.

\paragraph{Double Pendulum System}
\label{sec:double_pendulum}

We next consider the nonlinear Acrobot system, consisting of two
serially connected rigid links.
Following the standard Acrobot dynamics, the state is represented as

$$
    X_t =
    (\theta_1,\theta_2,\dot{\theta}_1,\dot{\theta}_2)_t,
$$

where $\theta_1$ and $\theta_2$ denote the two joint angles. We generate
the dataset by numerically simulating the corresponding equations of
motion and introduce additional stochastic perturbations to the angular
accelerations,
\begin{equation}
\begin{aligned}
    d\theta_1 &= \dot{\theta}_1\,dt, \qquad
    d\theta_2 = \dot{\theta}_2\,dt, \\
    d\dot{\theta}_1
    &= f_1(\theta_1,\theta_2,\dot{\theta}_1,\dot{\theta}_2)\,dt
    + \sigma_1\,dB_t^{(1)}, \\
    d\dot{\theta}_2
    &= f_2(\theta_1,\theta_2,\dot{\theta}_1,\dot{\theta}_2)\,dt
    + \sigma_2\,dB_t^{(2)},
\end{aligned}
\label{eq:acrobot}
\end{equation}
where $f_1$ and $f_2$ are determined by the Acrobot equations of motion
and $\sigma_1,\sigma_2$ control the strength of the artificial noise.
For the state-space experiment, we simulate 1000 trajectories, each
lasting 8 seconds at 30 frames per second, resulting in 240 observations
per trajectory. Initial joint angles are sampled uniformly from
$[-\pi,\pi]$, while the initial angular velocities are sampled from
$[-0.1,0.1]$. In both experiments, we use an 80/20 training--test split. Additional details of the physical parameters, numerical
simulation, and training procedure are provided in Appendix~\ref{app:db}.

\begin{figure}[!ht]
\centering
\includegraphics[width=\linewidth]
{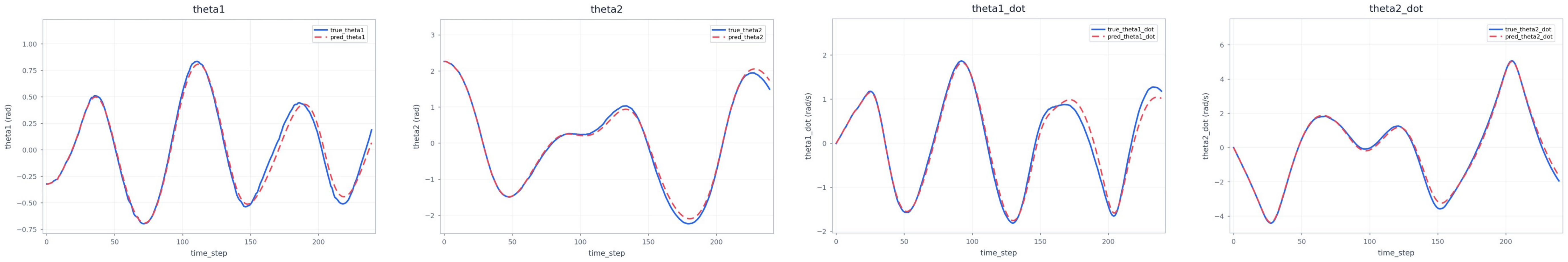}
\caption{
\textbf{State-space learning on the stochastic double-pendulum
system.}
Ground-truth and PhiBE-Flow predictions for two joint angles
$\theta_1,\theta_2$ and angular velocities
$\dot{\theta}_1,\dot{\theta}_2$ over physical time.
}
\label{fig:pendulum_state}
\end{figure}

We first apply PhiBE-Flow directly in the state space
and use a residual MLP to learn the probability velocity from consecutive
observations. Figure~\ref{fig:pendulum_state} compares the
predicted trajectories with the ground truth. PhiBE-Flow closely follows
the nonlinear evolution of both joint angles and their corresponding
angular velocities over the generation horizon.

\begin{figure}[!ht]
\centering
\includegraphics[width=\linewidth]
{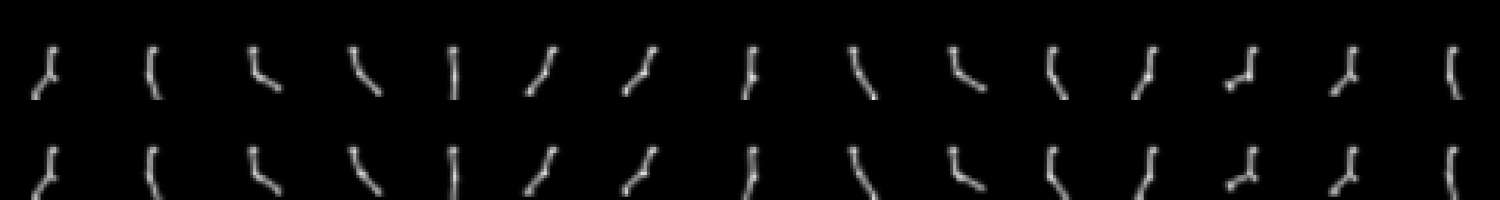}
\caption{
\textbf{Image-based Learning on the stochastic double-pendulum
system.}
The top row shows the ground-truth image sequence. The bottom row
shows results from PhiBE-Flow.
}
\label{fig:pendulum_image}
\end{figure}

We further construct an image-based learning task by rendering each
pendulum state into a grayscale $28\times28$ image. A geometry-preserving
encoder maps each frame to a three-dimensional latent representation;
the latent positions are then augmented with finite-difference
velocities to form a six-dimensional dynamical state. PhiBE-Flow learns
the probability velocity in this latent space, and the predicted latent
states are decoded back to image space. 

As shown in
Figure~\ref{fig:pendulum_image}, the predicted rollout closely follows
the ground-truth motion while preserving the geometry of the two-link
system. 
This experiment demonstrates that PhiBE-Flow can model the same
stochastic physical dynamics both from explicit states and from learned
visual representations.

\subsection{2D Navier--Stokes Equations Generation}
\label{sec:nse}

We next consider learning the stochastic two-dimensional Navier--Stokes equations on the torus in the vorticity formulation,
\begin{equation}
d\omega + \mathbf{v}\cdot\nabla\omega dt
= \nu\Delta\omega dt - \alpha\omega dt + \epsilon d\eta,
\end{equation}
where $\omega$ is the vorticity field, $\mathbf{v}$ is the corresponding velocity field, and $d\eta$ represents stochastic forcing. 

We use the PFI benchmark~\citep{chen2024probabilistic}, comprising 1,000 simulated trajectories and $2\times10^5$ single-channel vorticity snapshots on a $128\times128$ periodic grid, with an 80/20 training--test split. Given an observed vorticity field, the objective is to generate its future evolution over multiple time steps. PhiBE-Flow is compared against MCVD and PFI; additional dataset details are provided in Appendix~\ref{app:nse}.

{
\begin{figure}[!ht]
    \centering
    \includegraphics[
        width=\linewidth,
        trim=0 8 0 3,
        clip
    ]{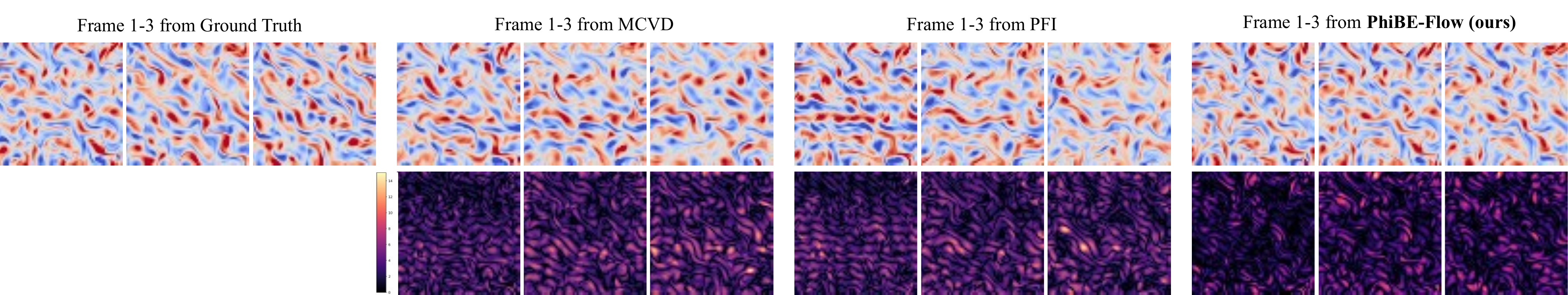}
    \caption{
    \textbf{Temporal generation results on Navier--Stokes.}
    Ground-truth vorticity fields and results from MCVD, PFI, and
    PhiBE-Flow, together with their corresponding error maps.
    }
    \label{fig:nse_vis}
\end{figure}
}

We evaluate Navier--Stokes generation results using the enstrophy spectrum, 
which can evaluate whether the results preserve the distribution of vorticity
across spatial scales.
For a vorticity field $\omega$, we compute its Fourier transform
$\widehat{\omega}$ and define the spectral enstrophy as
\[
    E(\mathbf{k})
    =
    \frac{1}{2}
    \left|
        \widehat{\omega}(\mathbf{k})
    \right|^2.
\]
The spectrum is obtained by averaging $E(\mathbf{k})$ over radial
wavenumber shells and subsequently averaging the resulting spectra over
the evaluation samples. The same procedure is applied to the
ground-truth fields and generation results from all compared methods.

\noindent
\begin{minipage}[t]{0.56\textwidth}
\vspace{0pt}
As shown in Figure~\ref{fig:nse_vis}, PhiBE-Flow more accurately captures
the evolution of the vorticity field and preserves fine-scale spatial
structures over the generation horizon. Its error maps exhibit smaller
deviations from the ground truth than those of MCVD and PFI.

The difference is also reflected in the enstrophy spectrum
(Figure~\ref{fig:nse_enstrophy}), which characterizes the distribution of
vorticity across spatial scales. PhiBE-Flow remains closer to the
ground-truth spectrum over a broad range of wavenumbers, including the
high-frequency regime associated with small-scale flow structures.
In contrast, MCVD exhibits a stronger loss of high-frequency energy,
while PFI shows larger deviations at small scales. These results indicate
that PhiBE-Flow better preserves both the generated flow structures and
their multiscale physical statistics.
\end{minipage}
\hfill
\begin{minipage}[t]{0.40\textwidth}
\vspace{0pt}
\centering
\includegraphics[
    width=\linewidth,
    trim=0 5 0 5,
    clip
]{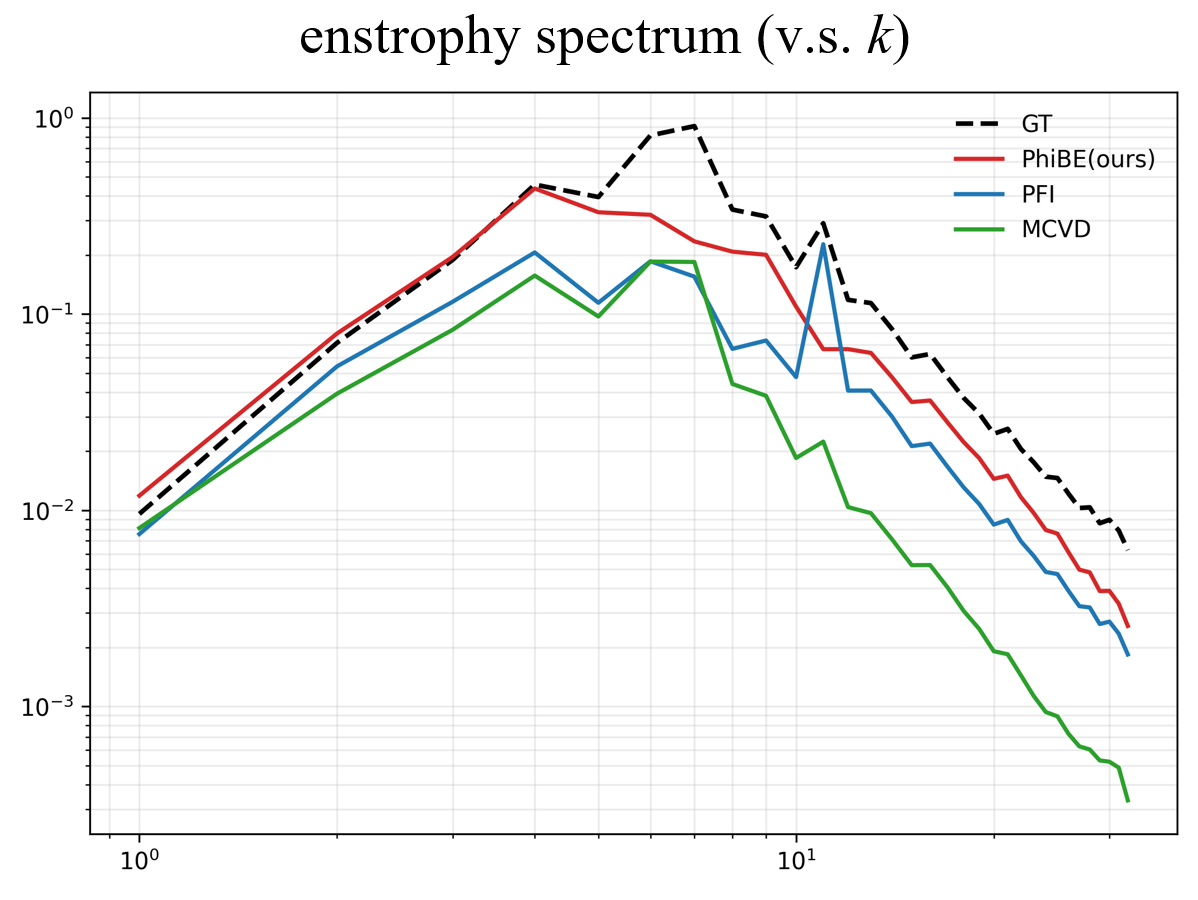}
\captionof{figure}{
Enstrophy spectrum comparison on the Navier--Stokes dataset.
}
\label{fig:nse_enstrophy}
\end{minipage}

\subsection{Video Generation}
\label{sec:kth}

Finally, we evaluate PhiBE-Flow on the KTH human action dataset~\citep{schuldt2004recognizing}, which contains 2,391 video sequences covering six action categories. We resize frames to $64\times64$ pixels and use an 80/20 training--test split.
Following latent video generation approaches, video frames are encoded
into a compact latent representation using a pretrained VQ-VAE, and
PhiBE-Flow models their temporal evolution in the latent space. The
predicted latent states are subsequently decoded back into video frames.
We compare against MCVD and PFI and evaluate the results using mean squared error (MSE),
Fr\'echet Video Distance (FVD)~\citep{unterthiner2018towards},
peak signal-to-noise ratio (PSNR), and structural similarity
(SSIM)~\citep{wang2004image}. Details of the representation and evaluation
protocol are provided in Appendix~\ref{app:kth}.
\begin{table}[!ht]
\centering
\caption{
\textbf{Metric results on KTH.}
Standard metrics compare generated results with the original ground
truth. Adjusted metrics compare generated results with the
reconstructed ground truth.
}
\label{tab:kth}

\setlength{\tabcolsep}{3.5pt}
\small
\begin{tabular}{lcccccccc}
\toprule

& \multicolumn{4}{c}{Standard metrics}
& \multicolumn{4}{c}{Adjusted metrics} \\

\cmidrule(lr){2-5}
\cmidrule(lr){6-9}

\textit{Method}
& MSE($\times10^{-2}$)$\downarrow$
& FVD$\downarrow$
& PSNR$\uparrow$
& SSIM$\uparrow$
& MSE($\times10^{-2}$)$\downarrow$
& FVD$\downarrow$
& PSNR$\uparrow$
& SSIM$\uparrow$ \\

\midrule

MCVD
& 0.8550
& 307.3
& 20.68
& 0.6986
& --
& --
& --
& -- \\

PFI
& 0.6637
& 280.1
& 21.78
& 0.7254
& 0.6576
& 247.1
& 21.82
& 0.7372 \\

\textbf{PhiBE-Flow}
& \textbf{0.4405}
& \textbf{227.5}
& \textbf{23.56}
& \textbf{0.7663}
& \textbf{0.4305}
& \textbf{192.7}
& \textbf{23.66}
& \textbf{0.7800} \\

\bottomrule
\end{tabular}
\end{table}

As shown in Table~\ref{tab:kth}, PhiBE-Flow achieves the best
performance across all four metrics. In particular, it substantially
reduces both MSE and FVD while improving PSNR and SSIM compared with
MCVD and PFI.
The qualitative results in Figure~\ref{fig:kth_comp2} are consistent
with the quantitative comparison. PhiBE-Flow better preserves the
human-body structure and motion over the generation horizon and
exhibits less accumulated motion drift than MCVD and PFI. These results
suggest that directly learning the physical-time probability velocity
provides a stable representation of temporal evolution even for
high-dimensional visual observations.

\begin{figure}[!ht]
    \centering
    \includegraphics[width=\linewidth]{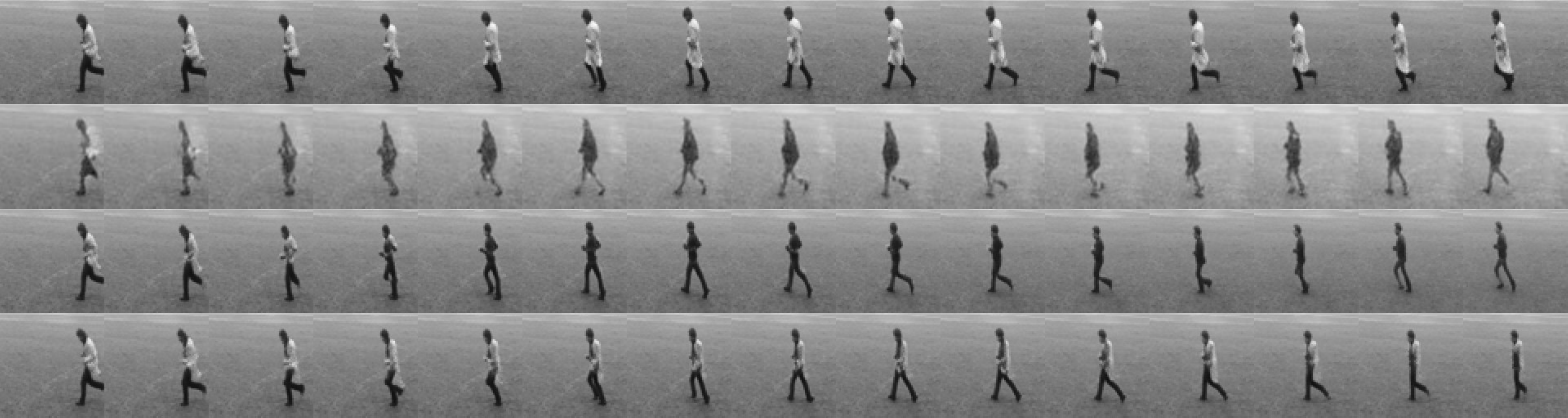}
    \caption{
    \textbf{Video generation on the KTH dataset.}
    The top row shows the ground-truth trajectory. The following three
    rows show generation results from MCVD, PFI, and PhiBE-Flow, respectively.
    }
    \label{fig:kth_comp2}
\end{figure}

\section{Conclusion}
We introduced \textbf{PhiBE-Flow}, a model-free framework that learns the SDE-induced probability velocity field from discrete observations to generate continuous-time trajectories consistent with density evolution. We established convergence guarantees and demonstrated its effectiveness on low-dimensional stochastic systems, Navier--Stokes dynamics, and video generation.

\textbf{Limitations and Future Work.}
PhiBE-Flow assumes that the observed dynamics admit an underlying SDE
representation, which may not hold for all real-world temporal processes.
Moreover, directly learning the probability velocity in high-dimensional
observation spaces can be computationally expensive, motivating the use of
latent representations. The resulting performance therefore depends on the
encoder--decoder, whose latent space should preserve the essential geometric
and dynamical structure of the original data. Future work will extend the
framework and its theoretical analysis to more general stochastic dynamics and more complex real-world
systems.

\bibliographystyle{plainnat}
\bibliography{references}  

\appendix
\section{Assumptions and Proofs}
\label{app:theory}

\subsection{Population Conditions and Objective Identity}
\label{app:regularity}

All spatial integrals are over $\mathbb R^d$. Uniform population
estimates are on $\mathcal Z=\Omega\times[0,T]$, where $\Omega$ is
compact and convex. Vector norms are Euclidean, matrix norms are
operator norms unless indicated otherwise, and $A:B$ denotes the
Frobenius inner product. We use
$(\nabla\cdot A)_i=\sum_j\partial_{x_j}A_{ij}$ and
$(\nabla u)_{ij}=\partial_{x_j}u_i$.

Using the same marginal density $\rho_t$ as in
\eqref{eq:probability_velocity}, define the population PhiBE field
\begin{equation}
\widehat v(x,t)=\widehat b(x)
-\frac12\widehat\Sigma(x)\nabla_x\log\rho_t(x)
-\frac12\nabla_x\cdot\widehat\Sigma(x).
\label{eq:population_phibe_field}
\end{equation}

\paragraph{Population conditions.}
The SDE is nonexplosive with locally smooth coefficients and is
defined through $T+h_0$ for some $h_0>0$.
Its density $\rho_t$ is positive, and
$\rho,\nabla_x\rho,\nabla_x^2\rho$ are continuous near $\mathcal Z$.
The fields $v$ and $\widehat v$ are locally Lipschitz in $x$,
uniformly in $t\in[0,T]$ for each fixed $h$.
The second-order semigroup expansions used in
Appendix~\ref{app:phibe_consistency} hold uniformly on $\Omega$,
with $o(\Delta t^2)$ remainders, including one spatial derivative of the
increment second moment. The coefficient functions appearing in
these expansions and their stated derivatives are continuous near
$\Omega$. These are assumptions on the short-time expansions;
local smoothness alone is not used to claim uniform remainder bounds.

\paragraph{Integrability and integration by parts.}
For each relevant $t$, assume
$\widehat b,\widehat\Sigma,\widehat v\in L^2(\rho_t)$ and
$\rho_t,\widehat\Sigma\in C^1_{\mathrm{loc}}(\mathbb R^d)$.
Admissible fields $u\in C^1_{\mathrm{loc}}$ satisfy
$u\in L^2(\rho_t)$ and
$\widehat\Sigma:\nabla u\in L^1(\rho_t)$.
Assume the analogous integrability for $b,\Sigma,v$ when using
the exact score-free objective, and time integrability when
integrating these identities over $[0,T]$.
Weak spatial derivatives are also allowed whenever the following
integration-by-parts identity remains valid.

For a fixed $t$, define
\[
\mathcal J_{h,t}(u)
=\int_{\mathbb R^d}
\bigl[\|u\|^2-2u\cdot\widehat b
-\widehat\Sigma:\nabla u\bigr]\rho_t\,dx.
\]
Integration by parts gives
\[
-\int (\widehat\Sigma:\nabla u)\rho_t\,dx
=\int u\cdot
(\widehat\Sigma\nabla\log\rho_t+
\nabla\cdot\widehat\Sigma)\rho_t\,dx.
\]
For noncompactly supported $u$, apply this first to $\chi_Ru$,
where $\chi_R$ is a smooth cutoff with
$\|\nabla\chi_R\|_\infty\le C/R$. The extra term is at most
\[
\frac C R\int\|u\|\,\|\widehat\Sigma\|_{\mathrm F}\rho_t\,dx
\le\frac C R\|u\|_{L^2(\rho_t)}
\|\widehat\Sigma\|_{L^2(\rho_t;\mathrm F)}\longrightarrow0.
\]
The other terms converge by the stated integrability assumptions.
Completing the square therefore yields
\begin{equation}
\mathcal J_{h,t}(u)
=\|u-\widehat v\|_{L^2(\rho_t)}^2
-\|\widehat v\|_{L^2(\rho_t)}^2.
\label{eq:population_objective_identity}
\end{equation}
The identity extends continuously from smooth compactly supported
fields to $L^2(\rho_t)$, with unique minimizer $\widehat v$.
Here $\widehat v$ uses the true density $\rho_t$ as in
\eqref{eq:population_phibe_field}.

\subsection{Proof of the PhiBE Coefficient Approximation}
\label{app:phibe_consistency}

\begin{proof}[Proof of Lemma~\ref{lem:phibe_consistency}]
Let $\mathcal L f=b\cdot\nabla f+\frac12\Sigma:\nabla^2f$.
For $f_j(y)=y_j$, the assumed semigroup expansion gives
\[
\widehat b_{j}(x)
=\frac{\mathbb E[f_j(X_{t+\Delta t})\mid X_t=x]-f_j(x)}{\Delta t}
=b_j(x)+\frac {\Delta t}2\mathcal L b_j(x)+o(\Delta t)
\]
uniformly on $\Omega$. Boundedness of $\mathcal L b$ on $\Omega$
and a sufficiently large $L_b$ give the first inequality for small $h$.

For $g_{jk,x}(y)=(y_j-x_j)(y_k-x_k)$, with $x$ held fixed when
applying $\mathcal L$ in $y$, we have
$g_{jk,x}(x)=0$ and $\mathcal Lg_{jk,x}(x)=\Sigma_{jk}(x)$.
Consequently,
\[
\widehat\Sigma_{jk}(x)
=\Sigma_{jk}(x)+\Delta tR_{jk}(x)+o(\Delta t),
\qquad R_{jk}(x):=\frac12\mathcal L^2g_{jk,x}(x).
\]
The expansion and its first spatial derivative hold uniformly
on $\Omega$ by assumption. Thus
\[
\|\Sigma-\widehat\Sigma\|_{\infty,\Omega}
\le \Delta t\|R\|_{\infty,\Omega}+o(\Delta t),
\qquad
\|\nabla\cdot(\Sigma-\widehat\Sigma)\|_{\infty,\Omega}
\le \Delta t\|\nabla\cdot R\|_{\infty,\Omega}+o(\Delta t).
\]
This proves the remaining inequalities.
\end{proof}

\subsection{Proof of the General Population Approximation}
\label{app:proof_general_population}

\begin{proof}[Proof of Theorem~\ref{thm:general_population}]
Positivity and continuity on $\mathcal Z$ imply
$\inf_{\mathcal Z}\rho_t>0$. Hence
\[
S_1:=\sup_{\mathcal Z}\|\nabla\log\rho_t\|<\infty,
\qquad
S_2:=\sup_{\mathcal Z}\|\nabla^2\log\rho_t\|<\infty.
\]
Subtracting \eqref{eq:population_phibe_field} from
\eqref{eq:probability_velocity} and applying
Lemma~\ref{lem:phibe_consistency} gives
\[
\|v-\widehat v\|_{\infty,\mathcal Z}
\le\left(L_b+\frac12L_\Sigma S_1+
\frac12L_{\mathrm{div}}\right)\Delta t+o(\Delta t).
\]
Take $L_v=L_b+\frac12L_\Sigma S_1+\frac12L_{\mathrm{div}}$.
The spatial derivatives of $v$ are bounded on $\mathcal Z$;
for example, a positive Lipschitz constant $K$ can be chosen at
least as large as
\[
\|\nabla b\|_{\infty,\Omega}
+\frac12\|\nabla\Sigma\|_{\infty,\Omega}S_1
+\frac12\|\Sigma\|_{\infty,\Omega}S_2
+\frac12\|\nabla(\nabla\cdot\Sigma)\|_{\infty,\Omega}.
\]
Here the norm of $\nabla\Sigma$ is the induced tensor norm.
Convexity of $\Omega$ gives
$\|v(x,t)-v(y,t)\|\le K\|x-y\|$ for $x,y\in\Omega$.
As long as both flows remain in $\Omega$, their common initial
condition implies
\[
\|x_t-\widehat x_t\|
\le K\int_0^t\|x_s-\widehat x_s\|\,ds
+t\bigl(L_v\Delta t+o(\Delta t)\bigr).
\]
Gr\"onwall's inequality yields
\[
\|x_t-\widehat x_t\|
\le\frac{e^{Kt}-1}{K}\bigl(L_v\Delta t+o(\Delta t)\bigr),
\]
uniformly for the stated times $t\le T$.
\end{proof}

\subsection{Constant Isotropic Diffusion}
\label{app:constant_scalar}

If $\Sigma=\sigma^2I_d$ is known, then
$\Sigma:\nabla u=\sigma^2\nabla\cdot u$ in the score-free
objective. The same simplification applies to a covariance
approximation constrained to be constant and isotropic,
$\widetilde\Sigma=\widehat\Sigma^{\,2}I_d$.
If
\[
|\sigma^2-\widehat\Sigma^{\,2}|
\le L_{\sigma^2}\Delta t+o(\Delta t),
\]
then, with
$\widetilde v=\widehat b-\frac12\widehat\Sigma^{\,2}
\nabla\log\rho_t$, the proof above gives
\[
\|v-\widetilde v\|_{\infty,\mathcal Z}
\le\left(L_b+\frac12L_{\sigma^2}S_1\right)\Delta t+o(\Delta t).
\]
The same flow bound follows under the same containment condition,
using any positive $K$ at least
$\|\nabla b\|_{\infty,\Omega}+\frac12\sigma^2S_2$.
The raw increment second moment $\widehat\Sigma(x)$ in
\eqref{eq:phibe_diffusion} need not be constant even when $\Sigma$
is constant, so this simplification does not automatically apply
to the empirical objective~\eqref{eq:empirical_phibe_objective}.

\subsection{Ornstein--Uhlenbeck Population Approximation}
\label{app:proof_ou_population}

\begin{corollary}[OU population approximation]
\label{cor:ou_population}
Consider
\[
dX_t=AX_t\,dt+\Gamma\,dB_t,\qquad
\Sigma:=\Gamma\Gamma^\top\succ0,
\qquad X_0\sim\mathcal N(m_0,C_0),\quad C_0\succ0,
\]
with $X_0$ independent of $B$. Let $\Omega\subset\mathbb R^d$ be
compact and $h=\Delta t$. The first-order PhiBE velocity is
\[
\widehat v(x,t)=\widehat Ax
-\frac12\nabla\cdot\widehat\Sigma(x)
+\frac12\widehat\Sigma(x)C_t^{-1}(x-m_t),
\]
where
\[
\begin{aligned}
M_h&=e^{Ah}-I,\qquad \widehat A=M_h/h,\\
Q_h&=\int_0^h e^{As}\Sigma e^{A^\top s}\,ds,\qquad
\widehat\Sigma(x)=\frac{Q_h+M_hxx^\top M_h^\top}{h},
\end{aligned}
\]
and $m_t,C_t$ are the marginal mean and covariance. There are
constants $L_{\mathrm{OU}},K_{\mathrm{OU}}>0$, independent of $h$,
such that, uniformly on $\Omega\times[0,T]$ as $h\downarrow0$,
\[
\|v(x,t)-\widehat v(x,t)\|
\le L_{\mathrm{OU}}h+O(h^2).
\]
For solutions of $\dot x_t=v(x_t,t)$ and
$\dot{\widehat x}_t=\widehat v(\widehat x_t,t)$ with the same initial
state, while both remain in $\Omega$,
\[
\|x_t-\widehat x_t\|
\le\frac{e^{K_{\mathrm{OU}}t}-1}{K_{\mathrm{OU}}}
L_{\mathrm{OU}}h+O(h^2),\qquad 0\le t\le T.
\]
\end{corollary}

\begin{proof}
The Gaussian marginals have
\[
m_t=e^{At}m_0,\qquad
C_t=e^{At}C_0e^{A^\top t}
+\int_0^t e^{As}\Sigma e^{A^\top s}\,ds.
\]
Thus $\nabla\log\rho_t(x)=-C_t^{-1}(x-m_t)$ and
$v(x,t)=Ax+\tfrac12\Sigma C_t^{-1}(x-m_t)$.
Conditionally on $X_t=x$, the increment is $M_hx+\xi_h$, where
$\xi_h\sim\mathcal N(0,Q_h)$. Its first and raw second moments give
$\widehat b(x)=\widehat Ax$ and $\widehat\Sigma$ above.
Direct differentiation yields
\[
\nabla\cdot\widehat\Sigma(x)
=\frac{M_h^2x+\operatorname{tr}(M_h)M_hx}{h},
\]
which also gives the stated expression for $\widehat v$.

Matrix exponential expansion gives, uniformly on $\Omega$ and also
after the spatial derivative used below,
\[
\widehat A=A+\tfrac12A^2h+O(h^2),\qquad
\widehat\Sigma(x)=\Sigma+hH(x)+O(h^2),
\]
where
\[
H(x)=\tfrac12(A\Sigma+\Sigma A^\top)+Axx^\top A^\top,
\qquad
\nabla\cdot H(x)=(A^2+\operatorname{tr}(A)A)x.
\]
Consequently,
\[
v(x,t)-\widehat v(x,t)
=\frac h2\left[-A^2x+\nabla\cdot H(x)
-H(x)C_t^{-1}(x-m_t)\right]+O(h^2).
\]
Since $C_0\succ0$, $C_t^{-1}$ is uniformly bounded on $[0,T]$.
With
\[
R=\sup_{x\in\Omega}\|x\|,\qquad
S_{\mathrm{OU}}=\sup_{(x,t)\in\Omega\times[0,T]}
\|C_t^{-1}(x-m_t)\|,
\]
the velocity bound follows by taking $L_{\mathrm{OU}}$ at least
\[
\tfrac12R\|A^2\|
+\tfrac12\|\nabla\cdot H\|_{\infty,\Omega}
+\tfrac12\|H\|_{\infty,\Omega}S_{\mathrm{OU}}.
\]
Finally, $v$ is spatially Lipschitz with any positive constant
\[
K_{\mathrm{OU}}\ge
\sup_{0\le t\le T}\|A+\tfrac12\Sigma C_t^{-1}\|.
\]
Subtracting the two ODEs and applying Gr\"onwall's inequality gives
the trajectory bound, with a remainder uniform for $t\in[0,T]$
while the trajectories remain in $\Omega$.
\end{proof}

\subsection{Linear-Basis Finite-Sample Setting and Proof}
\label{app:finite_sample_setting}

\paragraph{Independent trajectories and the sampling measure.}
Let $\Xi_k=(X_{t_0}^{(k)},\ldots,X_{t_L}^{(k)})$, $k=1,\ldots,N$,
be independent copies of an SDE trajectory observed at
$t_j=jh$, where $L=N_t\ge1$, $0<h=\Delta t\le1$, and $Lh\le T$.
Transitions within a trajectory are not assumed independent.
Write $W_j^{(k)}=X_{t_{j+1}}^{(k)}-X_{t_j}^{(k)}$ and
\[
\mu_L(dx,dt)=\frac1L\sum_{j=0}^{L-1}
\rho_{t_j}(x)\,dx\,\delta_{t_j}(dt).
\]
For one trajectory $\Xi$, define
\[
\ell_\Xi(u)=\frac1L\sum_{j=0}^{L-1}
\left[\|u(X_{t_j},t_j)\|^2
-2u(X_{t_j},t_j)\cdot\frac{W_j}{h}
-\frac{W_jW_j^\top}{h}:\nabla_xu(X_{t_j},t_j)\right].
\]
Set $\mathcal J_N(u)=N^{-1}\sum_k\ell_{\Xi_k}(u)$ and
$\mathcal J(u)=\mathbb E\ell_\Xi(u)$. This is the uniform-grid
empirical objective~\eqref{eq:empirical_phibe_objective}.
Conditional expectation and the integration-by-parts identity in
Appendix~\ref{app:regularity} give
\begin{equation}
\label{eq:finite_population_identity}
\mathcal J(u)=\|u-\widehat v\|_{L^2(\mu_L)}^2
-\|\widehat v\|_{L^2(\mu_L)}^2.
\end{equation}
This objective averages over the observation times; no time
quadrature claim is needed for this identity.

\paragraph{Fixed linear-basis models.}
Let $\psi_1,\ldots,\psi_p:\mathbb R^d\times[0,T]\to\mathbb R^d$
be deterministic basis fields and $h=\Delta t$, fixed independently of the training
sample, for which the derivatives and expectations below are
well defined. Write
\[
\Psi(x,t)=
\begin{pmatrix}\psi_1(x,t)&\cdots&\psi_p(x,t)\end{pmatrix},
\qquad
\mathcal V_p=\{v_\theta=\Psi\theta:\theta\in\mathbb R^p\}.
\]
The basis fields may be nonlinear in $(x,t)$; only their
coefficients are fitted.
For one trajectory, define
\[
H_\Xi=\frac1L\sum_{j=0}^{L-1}
\Psi(X_{t_j},t_j)^\top\Psi(X_{t_j},t_j),
\]
and, for $\ell=1,\ldots,p$,
\[
(c_\Xi)_\ell=\frac1L\sum_{j=0}^{L-1}
\left[
\psi_\ell(X_{t_j},t_j)\cdot\frac{W_j}{h}
+\frac{W_jW_j^\top}{2h}:
\nabla_x\psi_\ell(X_{t_j},t_j)
\right].
\]
Set
\[
H=\mathbb E H_\Xi,\qquad c=\mathbb E c_\Xi,
\qquad
H_N=\frac1N\sum_{k=1}^NH_{\Xi_k},
\qquad
c_N=\frac1N\sum_{k=1}^Nc_{\Xi_k}.
\]
The empirical and population objectives satisfy
\[
\mathcal J_N(v_\theta)=\theta^\top H_N\theta-2c_N^\top\theta,
\qquad
\mathcal J(v_\theta)=\theta^\top H\theta-2c^\top\theta.
\]
When $H\succ0$, let $\theta_\star=H^{-1}c$.
Then $\Psi\theta_\star$ is the population minimizer in
$\mathcal V_p$. By \eqref{eq:finite_population_identity},
it equals $\widehat v$ under the realizability condition
$\widehat v\in\mathcal V_p$.
Define
\[
\theta_N=
\begin{cases}
H_N^{-1}c_N,& H_N\succ0,\\
0,&\text{otherwise},
\end{cases}
\qquad
v_N=\Psi\theta_N.
\]
On the event $H_N\succ0$, this is the unique global empirical
minimizer in $\mathcal V_p$.
For a scalar random variable $U$, write
\[
\|U\|_{\psi_1}
:=\inf\{a>0:\mathbb E\exp(|U|/a)\le2\}.
\]

\begin{assumption}[Linear-basis sampling conditions]
\label{ass:finite_sample_class}
In the preceding independent-trajectory setting, the identity
\eqref{eq:finite_population_identity} holds on $\mathcal V_p$,
$H\succ0$, and $\widehat v\in\mathcal V_p$.
Let $h=\Delta t$.
Define
\[
G_\Xi=H^{-1/2}H_\Xi H^{-1/2},
\qquad
\eta_\Xi=H^{-1/2}(c_\Xi-H_\Xi\theta_\star).
\]
There exist constants $K_G,K_R>0$, independent of $N$ and
uniform over the observation grids under consideration, such that
\[
\sup_{\|a\|_2=1}
\|a^\top(G_\Xi-I_p)a\|_{\psi_1}\le K_G,
\qquad
\max_{1\le\ell\le p}\|(\eta_\Xi)_\ell\|_{\psi_1}
\le\frac{K_R}{\sqrt h}.
\]
\end{assumption}

The moment conditions apply to the chosen basis and trajectory
distribution; linearity alone does not imply them. Their constants
may depend on the fixed basis and its conditioning. No uniform
claim over growing basis families is made without uniform bounds.

\begin{proof}[Proof of Theorem~\ref{thm:general_finite_sample}]
The definitions give
\[
\mathbb E G_\Xi=I_p,\qquad
\mathbb E\eta_\Xi=H^{-1/2}(c-H\theta_\star)=0.
\]
Put
\[
G_N=H^{-1/2}H_NH^{-1/2},
\qquad
\overline\eta_N=\frac1N\sum_{k=1}^N\eta_{\Xi_k},
\qquad
\ell_\delta=\log(4p/\delta).
\]
All concentration arguments concern the $N$ independent
trajectory contributions; observations within a trajectory
need not be independent.

Choose a $1/4$-net $\mathcal A$ of the Euclidean unit sphere
in $\mathbb R^p$ with $|\mathcal A|\le9^p$.
Scalar Bernstein concentration and a union bound give, with
probability at least $1-\delta/2$,
\[
\max_{a\in\mathcal A}|a^\top(G_N-I_p)a|
\le CK_G
\left[
\sqrt{\frac{p+\log(4/\delta)}N}
+\frac{p+\log(4/\delta)}N
\right].
\]
For any symmetric matrix $B$, the net inequality gives
\[
\|B\|_{\mathrm{op}}
\le2\max_{a\in\mathcal A}|a^\top Ba|.
\]
Thus, if
\[
N\ge C(1+K_G+K_G^2)
\bigl[p+\log(4/\delta)\bigr],
\]
then $\|G_N-I_p\|_{\mathrm{op}}\le1/2$.
On this event, $H_N\succ0$ and
$\|G_N^{-1}\|_{\mathrm{op}}\le2$.

Applying scalar Bernstein concentration to each coordinate of
$\overline\eta_N$ and taking a union bound gives, with probability
at least $1-\delta/2$,
\[
\|\overline\eta_N\|_2
\le CK_R\sqrt{\frac ph}
\left[
\sqrt{\frac{\ell_\delta}N}
+\frac{\ell_\delta}N
\right].
\]
On the intersection of these events, the empirical normal equations
imply
\[
H_N(\theta_N-\theta_\star)=c_N-H_N\theta_\star,
\]
and therefore
\[
G_NH^{1/2}(\theta_N-\theta_\star)=\overline\eta_N.
\]
By realizability and the definition of $H$,
\[
\|v_N-\widehat v\|_{L^2(\mu_L)}^2
=(\theta_N-\theta_\star)^\top
H(\theta_N-\theta_\star).
\]
Consequently,
\[
\|v_N-\widehat v\|_{L^2(\mu_L)}
\le2\|\overline\eta_N\|_2
\le CK_R\sqrt{\frac ph}
\left[
\sqrt{\frac{\ell_\delta}N}
+\frac{\ell_\delta}N
\right].
\]
For sufficiently large $c_0$ depending only on $K_G$,
$N\ge c_0p\log(2p/\delta)$ ensures both the Gram-matrix condition
and $N\ge\ell_\delta$.
The latter gives
$\ell_\delta/N\le\sqrt{\ell_\delta/N}$.
Since $\ell_\delta\le2\log(2p/\delta)$, this proves
\[
\|v_N-\widehat v\|_{L^2(\mu_L)}
\le C_{\mathcal V}
\sqrt{\frac{p\log(2p/\delta)}{Nh}}
\]
with probability at least $1-\delta$.
\end{proof}

\paragraph{When the basis is not realizable.}
If $\widehat v\notin\mathcal V_p$, the population minimizer
$v_p=\Psi\theta_\star$ is its $L^2(\mu_L)$ projection onto
$\mathcal V_p$. Under the other conditions of
Assumption~\ref{ass:finite_sample_class}, with the residual
$\eta_\Xi$ defined using this $\theta_\star$, the same proof bounds
$\|v_N-v_p\|_{L^2(\mu_L)}$. Consequently, with the same probability,
\[
\|v_N-\widehat v\|_{L^2(\mu_L)}
\le \inf_{u\in\mathcal V_p}\|u-\widehat v\|_{L^2(\mu_L)}
+C_{\mathcal V}\sqrt{\frac{p\log(2p/\delta)}{Nh}}.
\]
Thus freezing neural features gives the required linear structure,
but does not by itself establish realizability or the moment bounds.

\subsection{Ornstein--Uhlenbeck Finite-Sample Convergence}
\label{app:proof_ou_finite_sample}

Consider
\[
dX_t=A_{\mathrm{OU}}(X_t-m)\,dt+S\,dB_t,
\qquad X_0\sim\mathcal N(m_0,C_0),\quad C_0\succ0,
\]
with $X_0$ independent of $B$. Observe $N$ independent trajectories
at $t_j=jh$, $j=0,\ldots,L$, where $0<h=\Delta t\le1$ and $Lh\le T$.
Write $W_j^{(k)}=X_{t_{j+1}}^{(k)}-X_{t_j}^{(k)}$.
At each $j<L$, fit $u_M(x)=M(x^\top,1)^\top$ by minimizing
\[
\mathcal J_{N,j}^{\mathrm{aff}}(M)
=\frac1N\sum_{k=1}^N
\left[
\|u_M(X_{t_j}^{(k)})\|^2
-2u_M(X_{t_j}^{(k)})\cdot\frac{W_j^{(k)}}h
-\frac{W_j^{(k)}W_j^{(k)\top}}h:\nabla_xu_M
\right].
\]
Write $v_{N,\mathrm{aff}}(\cdot,t_j)$ and
$\widehat v_{\mathrm{aff},h}(\cdot,t_j)$ for the empirical and
population affine minimizers. Gaussian integration by parts shows
that the latter is the $L^2(\rho_{t_j})$ projection of $\widehat v$
onto the affine class. It need not equal $\widehat v$: the raw
increment second moment is quadratic in $x$, so the unrestricted
PhiBE field can contain cubic terms. The compact set below is an
evaluation region; the Gaussian samples are not truncated to it.

\begin{corollary}[Affine OU convergence at observation times]
\label{cor:ou_finite_sample}
In the preceding setting, fix a compact set $\Omega\subset\mathbb R^d$
and put $p_{\mathrm{aff}}=d(d+1)$ and
$\ell_\delta=\log(2p_{\mathrm{aff}}L/\delta)$.
There exist constants $c_{\mathrm{OU}},C_{\Omega,\mathrm{OU}}>0$,
depending only on the OU parameters, $T,d,\Omega$, such that, if
$\delta\in(0,1)$ and
$N\ge c_{\mathrm{OU}}p_{\mathrm{aff}}\ell_\delta$, then, with
probability at least $1-\delta$,
\[
\max_{0\le j<L}\sup_{x\in\Omega}
\|v_{N,\mathrm{aff}}(x,t_j)
-\widehat v_{\mathrm{aff},h}(x,t_j)\|
\le C_{\Omega,\mathrm{OU}}
\sqrt{\frac{p_{\mathrm{aff}}\ell_\delta}{Nh}}.
\]
\end{corollary}

\begin{proof}
Put $q=d+1$, $z_j=(X_{t_j}^\top,1)^\top$, and $w_j=W_j$.
Define
\[
\begin{aligned}
H_{N,j}&=\frac1N\sum_{k=1}^N z_j^{(k)}z_j^{(k)\top},\\
D_{N,j}&=\frac1N\sum_{k=1}^N
\left[\frac{w_j^{(k)}z_j^{(k)\top}}h
+\left(\frac{w_j^{(k)}w_j^{(k)\top}}{2h},0\right)\right],
\end{aligned}
\]
where $(B,0)$ appends a zero column. Let $H_j=\mathbb EH_{N,j}$
and $D_j=\mathbb ED_{N,j}$. The objective is
$\operatorname{tr}(MH_{N,j}M^\top)-2\langle M,D_{N,j}\rangle_F$,
so its normal equations and their population counterparts are
\[
M_{N,j}H_{N,j}=D_{N,j},\qquad \widehat M_jH_j=D_j.
\]
The Gaussian marginal means and covariances are bounded on $[0,T]$.
Since $C_0\succ0$, their covariances are also uniformly positive
definite, and hence $\inf_{j,h}\lambda_{\min}(H_j)=:\lambda_*>0$.
For $N\ge q$, the augmented Gaussian design has full row rank almost
surely at all observation times; thus the empirical minimizers exist
and are unique. Choose $c_{\mathrm{OU}}$ to ensure $N\ge q$.

The exact transition representation is
\[
w_j=F_h(X_{t_j}-m)+\xi_j,\qquad F_h=e^{A_{\mathrm{OU}}h}-I,
\]
\[
\xi_j\sim\mathcal N\!\left(0,
\int_0^h e^{A_{\mathrm{OU}}s}SS^\top
e^{A_{\mathrm{OU}}^\top s}\,ds\right),
\]
where $\xi_j$ is independent of $X_{t_j}$. Consequently,
$\|F_h\|\le Ch$, the coordinates of $z_j$ have uniformly bounded
sub-Gaussian norms, and those of $w_j$ have sub-Gaussian norms at
most $C\sqrt h$. Products of sub-Gaussian variables are
sub-exponential even when the factors are dependent. Thus the
centered entries of $z_jz_j^\top$ have $\psi_1$-norm at most $C$,
and those of
$w_jz_j^\top/h+(w_jw_j^\top/(2h),0)$ have $\psi_1$-norm at most
$C/\sqrt h$. Here constants are independent of $j,N,h,L,\delta$.

Apply scalar Bernstein inequalities across the $N$ independent
trajectories and take a union bound over entries and times. With
probability at least $1-\delta$, simultaneously,
\[
\begin{aligned}
\max_{j<L}\|H_{N,j}-H_j\|_{\mathrm{op}}
&\le Cq\left(\sqrt{\frac{\ell_\delta}{N}}
+\frac{\ell_\delta}{N}\right),\\
\max_{j<L}\|D_{N,j}-D_j\|_F
&\le C\sqrt{\frac{p_{\mathrm{aff}}}{h}}
\left(\sqrt{\frac{\ell_\delta}{N}}
+\frac{\ell_\delta}{N}\right).
\end{aligned}
\]
Since $q^2\le2p_{\mathrm{aff}}$, the sample-size condition ensures
$H_{N,j}\succeq(\lambda_*/2)I_q$ for every $j<L$.
The transition representation also gives
$\sup_{j,h}\|D_j\|_F<\infty$: the term involving
$\mathbb E[\xi_jz_j^\top]$ vanishes, $F_h/h$ is bounded, and
$\mathbb E[w_jw_j^\top]/h$ is bounded. Therefore
$\widehat M_j=D_jH_j^{-1}$ is uniformly bounded.
Subtracting the normal equations yields
\[
M_{N,j}-\widehat M_j
=\bigl[D_{N,j}-D_j+\widehat M_j(H_j-H_{N,j})\bigr]H_{N,j}^{-1}.
\]
The preceding bounds, $h\le1$, and $q\le\sqrt{2p_{\mathrm{aff}}}$
give
\[
\max_{j<L}\|M_{N,j}-\widehat M_j\|_F
\le C\sqrt{\frac{p_{\mathrm{aff}}\ell_\delta}{Nh}}.
\]
Multiplication by
$\sup_{x\in\Omega}\|(x^\top,1)^\top\|<\infty$ proves the claim.
No independence between transitions within a trajectory is used.
\end{proof}

\subsection{Proof of the Overall Convergence}
\label{app:proof_overall_convergence}

\begin{proof}[Proof of Corollary~\ref{cor:overall_convergence}]
Restriction to $\Omega$ and the triangle inequality give
\[
\|v_N-v\|_{L^2(\mu_L;\Omega)}
\le\|v_N-\widehat v\|_{L^2(\mu_L)}
+\|\widehat v-v\|_{L^2(\mu_L;\Omega)}.
\]
Theorem~\ref{thm:general_finite_sample} bounds the first term.
Since $L^{-1}\sum_{j=0}^{L-1}\int_\Omega\rho_{t_j}(x)\,dx\le1$,
the second is at most
$\|\widehat v-v\|_{\infty,\Omega\times[0,T]}$ and is bounded by
Theorem~\ref{thm:general_population}. This proves the result.
\end{proof}

\section{Experimental Settings and More Results}
\label{app:B}
\subsection{Ornstein--Uhlenbeck Process}
\label{app:ou}

\paragraph{Dynamics and analytical solution.}
We consider the two-dimensional Ornstein--Uhlenbeck process
\begin{equation}
    dX_t
    =
    -\Gamma(X_t-\mu)\,dt
    +
    {\sigma}\,dB_t,
    \qquad
    X_0\sim\mathcal{N}(\mu_0,\Sigma_0),
    \label{eq:ou_sde_appendix}
\end{equation}
where $\Gamma\in\mathbb{R}^{2\times 2}, \; \sigma\in\mathbb{R}^{2\times 2}$, $\mu\in\mathbb{R}^2.$   We first establish a general result for the marginal distributions
of such OU process with Gaussian initial.
\begin{theorem}[Gaussian marginals of the OU process]
\label{thm:ou_gaussian_marginals}
Consider the Ornstein--Uhlenbeck process
\begin{equation}
    dX_t=\Gamma(X_t-\mu)\,dt+\sigma\,dW_t,
    \qquad
    X_0\sim\mathcal{N}(\mu_0,\Sigma_0),
    \label{eq:ou_general_sde}
\end{equation}
where $\Gamma\in\mathbb{R}^{d\times d}$ has eigenvalues with
strictly negative real parts,
$\sigma\in\mathbb{R}^{d\times r}$ is constant,
$\sigma\sigma^\top\succ0$, and $\Sigma_0\succ0$.
Assume that $X_0$ is independent of the Brownian motion $W$.

Then, for every $t\geq0$, the marginal distribution of $X_t$
is Gaussian, $\mathcal{N}(\mu_t,\Sigma_t)$, with
\begin{align}
    \mu_t
    &=
    \mu+\exp(\Gamma t)(\mu_0-\mu),
    \label{eq:ou_mean_theorem}
    \\
    \Sigma_t
    &=
    \exp(\Gamma t)(\Sigma_0-B)\exp(\Gamma^\top t)+B,
    \label{eq:ou_covariance_theorem}
\end{align}
where $B$ is the unique symmetric positive-definite solution
of the Lyapunov equation
\begin{equation}\label{sylvester_equ}
    \Gamma B+B\Gamma^\top=-\sigma\sigma^\top.
\end{equation}
In particular,
\[
    \mu_t\longrightarrow\mu,
    \qquad
    \Sigma_t\longrightarrow B,
    \qquad
    t\longrightarrow\infty,
\]
and the process has the unique stationary distribution
$\mathcal{N}(\mu,B)$.
\end{theorem}
We defer the proof to the end of Sec. \ref{app:ou}.

\textbf{Reversible case:} We set
\[
    \Gamma=\operatorname{diag}(1,4),
    \qquad
    \mu=(4,4)^\top,
    \qquad
    {\sigma}=\frac{1}{\sqrt{2}}I_2, 
    \qquad 
    \mu_0=(0,0)^\top,
    \qquad
    \Sigma_0=I_2.
\]

By Theorem \ref{thm:ou_gaussian_marginals}, the marginal distribution remains Gaussian,
\[
    \rho_t=\mathcal{N}(\mu_t,\Sigma_t),
\]
where
\begin{equation}
    \mu_t
    =
    \begin{bmatrix}
        4(1-e^{-t})\\
        4(1-e^{-4t})
    \end{bmatrix},
    \qquad
    \Sigma_t
    =
    \begin{bmatrix}
        e^{-2t}+\frac{1-e^{-2t}}{4} & 0\\
        0 &
        e^{-8t}+\frac{1-e^{-8t}}{16}
    \end{bmatrix}.
    \label{eq:ou_exact_marginal}
\end{equation}

Since
\[
    \nabla_x\log\rho_t(x)
    =
    -\Sigma_t^{-1}(x-\mu_t),
\]
the exact probability velocity under the SDE convention of
\eqref{eq:ou_sde_appendix} is
\begin{equation}
    v^\star(x,t)
    =
    -\Gamma(x-\mu)
    +
    \frac{1}{4}\Sigma_t^{-1}(x-\mu_t).
    \label{eq:ou_exact_velocity}
\end{equation}
We refer readers to \cite{pavliotis2014stochastic} for detailed derivation of the above formulation. Equation~\ref{eq:ou_exact_velocity} is used only for evaluation and is
not provided to the neural PhiBE-Flow model during training.

\textbf{Non-reversible case:}  We retain the values of $\mu$, $\mu_0$, and $\Sigma_0$
from the reversible case and choose
\[
    \Gamma=
    \begin{bmatrix}
        -2 & -1 \\
        -1 & -2
    \end{bmatrix},
    \qquad
    \sigma=
    \begin{bmatrix}
        1 & 0 \\
        0 & \frac{1}{\sqrt{2}}
    \end{bmatrix}.
\]
The eigenvalues of $\Gamma$ are $-1$ and $-3$, so the
stability condition in
Theorem~\ref{thm:ou_gaussian_marginals} is satisfied.

By Theorem~\ref{thm:ou_gaussian_marginals}, the marginal
density is Gaussian, and its score is
\[
    \nabla_x\log\rho(x,t)
    =-\Sigma_t^{-1}(x-\mu_t).
\]
The probability velocity associated with
\eqref{eq:ou_general_sde} is therefore
\begin{align}
    v(x,t)
    &=
    \Gamma(x-\mu)
    -\frac{1}{2}\sigma\sigma^\top\nabla_x\log\rho(x,t)
    \notag\\
    &=
    \Gamma(x-\mu)
    +\frac{1}{2}\sigma\sigma^\top
    \Sigma_t^{-1}(x-\mu_t),
    \label{eq:ou_nonreversible_velocity}
\end{align}
where $\mu_t$ and $\Sigma_t$ are given by
\eqref{eq:ou_mean_theorem} and
\eqref{eq:ou_covariance_theorem}, respectively.
The corresponding deterministic probability flow satisfies
\[
    \dot{x}_t=v(x_t,t).
\]

Substituting the expression for $\mu_t$ into
\eqref{eq:ou_nonreversible_velocity}, we obtain
\begin{align*}
    v(x,t)
    &=
    \left(
        \Gamma+\frac{1}{2}\sigma\sigma^\top\Sigma_t^{-1}
    \right)(x-\mu)
    +\frac{1}{2}\sigma\sigma^\top
    \Sigma_t^{-1}(\mu-\mu_t)
    \\
    &=
    \Gamma(x-\mu)
    +\frac{1}{2}\sigma\sigma^\top\Sigma_t^{-1}
    \left(
        x-\mu+\exp(\Gamma t)(\mu-\mu_0)
    \right).
\end{align*}

As $t\rightarrow\infty$, the marginal density converges to
\[
    \rho_\infty(x)
    =
    \frac{1}{\sqrt{(2\pi)^d\det(B)}}
    \exp\left(
        -\frac{1}{2}(x-\mu)^\top B^{-1}(x-\mu)
    \right),
\]
and the probability velocity converges to
\[
    v_\infty(x)=\mathscr{B}(x-\mu),
\]
where
\begin{equation}\label{limit B}
    \mathscr{B}
    =
    \Gamma+\frac{1}{2}\sigma\sigma^\top B^{-1}
    =
    \frac{1}{2}
    \left(
        \Gamma-B\Gamma^\top B^{-1}
    \right).
\end{equation}
The second equality follows from \eqref{sylvester_equ}.

The limiting velocity matrix satisfies
\[
    \mathscr{B}B+B\mathscr{B}^\top=0,
\]
or, equivalently,
\[
    \mathscr{B}^\top=-B^{-1}\mathscr{B}B.
\]
Hence, the transformed matrix
\[
    \widetilde{\mathscr{B}}
    =B^{-1/2}\mathscr{B}B^{1/2}
\]
is antisymmetric. The limiting probability flow therefore
preserves the quadratic form
\[
    \frac{\mathrm{d}}{\mathrm{d}t}
    \left(
        (x_t-\mu)^\top B^{-1}(x_t-\mu)
    \right)=0,
\]
and leaves the stationary Gaussian distribution invariant.
Thus, a stationary marginal distribution can coexist with
persistent probability circulation.

For the matrices $\Gamma$ and $\sigma$ specified above,
solving \eqref{sylvester_equ} yields
\[
    B=
    \frac{1}{16}
    \begin{bmatrix}
        5 & -2 \\
        -2 & 3
    \end{bmatrix},
    \qquad
    B^{-1}=
    \frac{16}{11}
    \begin{bmatrix}
        3 & 2 \\
        2 & 5
    \end{bmatrix}.
\]
Substituting these expressions into \eqref{limit B} gives
\begin{align*}
    \mathscr{B}
    &=
    \frac{1}{2}
    \left(
        \Gamma-B\Gamma^\top B^{-1}
    \right)
    \\
    &=
    \frac{1}{11}
    \begin{bmatrix}
        2 & 5 \\
        -3 & -2
    \end{bmatrix}.
\end{align*}
Consequently, the limiting probability-flow dynamics are
\[
    \dot{x}_t=\mathscr{B}(x_t-\mu).
\]
The nonconstant trajectories follow closed elliptic orbits
centered at $\mu$, with orthogonal principal-axis directions
$(\sqrt{5}-1,2)^\top$ and $(-(\sqrt{5}+1),2)^\top$.
Moreover,
\[
    \mathscr{B}^2=-\frac{1}{11}I,
\]
so the trajectories are periodic with angular frequency
$\omega=1/\sqrt{11}$ and period $2\pi\sqrt{11}$.

\paragraph{Trajectory generation.}
We generate independent sample trajectories from
\eqref{eq:ou_sde_appendix} and record the states at discrete physical-time
nodes
\[
    0=t_0<t_1<\cdots<t_{N_t-1}=T.
\]
The resulting training set consists of adjacent transition pairs
\[
    \left\{
        \left(
            X_{t_i}^{(k)},
            X_{t_{i+1}}^{(k)}
        \right)
    \right\}_{k=1,\ldots,N_x;\;
              i=0,\ldots,N_t-2}.
\]
For the neural velocity-field experiment, we use
\[
    T=1,\qquad
    N_x=50{,}000,\qquad
    N_t=2{,}000.
\]

\paragraph{Affine velocity field and optimization.}
As discussed in Sec.~\ref{sec:ou}, for the non-reversible
Ornstein--Uhlenbeck (OU) process, we seek an affine vector field
$v_{\theta_i}:\mathbb{R}^d \to \mathbb{R}^d$ at each time step $t_i$
that approximates
\[
\mathbf{b}(\cdot)
- \frac 1 2\bigl(\nabla \cdot \boldsymbol{\Sigma}(\cdot)\bigr)^\top
- \frac 1 2\boldsymbol{\Sigma}(\cdot)\nabla \log \rho_{t_i}(\cdot).
\]
For convenience, let $N=N_{\mathbf{x}}$,
$x_k=x_{t_i}^{(k)}$,
$\Delta x_k=x_{t_{i+1}}^{(k)}-x_{t_i}^{(k)}$, and
$\Delta t_i=t_{i+1}-t_i$.
The corresponding empirical least-squares problem is
\begin{equation}
\min_{\theta_i \in \mathbb{R}^m}
\frac{1}{N}\sum_{k=1}^{N}
\left[
    \lVert v_{\theta_i}(x_k) \rVert^2
    - 2v_{\theta_i}(x_k)\cdot
      \frac{\Delta x_k}{\Delta t_i}
    - 2\nabla_x v_{\theta_i}(x_k):
      \frac{\Delta x_k\Delta x_k^\top}{2\Delta t_i}
\right].
\label{lst_seq_vf_affine_Ito}
\end{equation}

For a fixed time step, we suppress the time index in the parameters
and write
\[
v_\theta(x)=Ax+\mu,
\qquad
A\in\mathbb{R}^{d\times d},
\quad
\mu\in\mathbb{R}^d,
\quad
\theta=(A,\mu).
\]
Introducing the augmented matrix and vector
\[
\widetilde{A}=[A\mid\mu]\in\mathbb{R}^{d\times(d+1)},
\qquad
\widetilde{x}=
\begin{bmatrix}
x\\
1
\end{bmatrix}
\in\mathbb{R}^{d+1},
\]
we have $v_\theta(x)=\widetilde{A}\widetilde{x}$.
Define the augmented sample matrix and the increment matrix as
\[
\widetilde{X}
=
[\widetilde{x}_1,\ldots,\widetilde{x}_N]
\in\mathbb{R}^{(d+1)\times N},
\qquad
\widetilde{x}_k=
\begin{bmatrix}
x_k\\
1
\end{bmatrix},
\qquad
\Delta X=[\Delta x_1,\ldots,\Delta x_N]
\in\mathbb{R}^{d\times N}.
\]
Let
\[
B=\frac{\Delta X}{\Delta t_i}\in\mathbb{R}^{d\times N},
\qquad
C=\frac{\Delta X\Delta X^\top}{2\Delta t_i}
\in\mathbb{R}^{d\times d},
\qquad
\widetilde{C}=
\begin{bmatrix}
C\\
0^\top
\end{bmatrix}
\in\mathbb{R}^{(d+1)\times d}.
\]
Since $\nabla_x v_\theta(x)=A$, the objective in
\eqref{lst_seq_vf_affine_Ito} can be written as
\[
\frac{1}{N}
\left[
    \lVert \widetilde{A}\widetilde{X} \rVert_F^2
    - 2\operatorname{Tr}
      \bigl(B^\top\widetilde{A}\widetilde{X}\bigr)
    - 2\operatorname{Tr}
      \bigl(\widetilde{A}\widetilde{C}\bigr)
\right].
\]
Thus, the minimization problem is equivalent to
\[
\min_{\widetilde{A}}
\operatorname{Tr}
\bigl(\widetilde{X}\widetilde{X}^\top
      \widetilde{A}^\top\widetilde{A}\bigr)
-2\operatorname{Tr}
\bigl(\widetilde{X}B^\top\widetilde{A}\bigr)
-2\operatorname{Tr}
\bigl(\widetilde{C}\widetilde{A}\bigr).
\]
Setting the gradient with respect to $\widetilde{A}$ to zero yields
the normal equations
\[
\widetilde{A}^{*}\widetilde{X}\widetilde{X}^\top
=
B\widetilde{X}^\top+\widetilde{C}^\top.
\]
Provided that these equations are consistent, the minimizer with
the smallest Frobenius norm is
\[
\widetilde{A}^{*}
=
\bigl(B\widetilde{X}^\top+\widetilde{C}^\top\bigr)
\bigl(\widetilde{X}\widetilde{X}^\top\bigr)^\dagger,
\]
where $\dagger$ denotes the Moore--Penrose pseudoinverse.
If $\widetilde{X}$ has full row rank, the minimizer is unique,
and the pseudoinverse reduces to the ordinary inverse.
The resulting optimal vector field at time $t_i$ is therefore
\[
v_{\theta_i^*}(x)=\widetilde{A}^{*}\widetilde{x},
\]
with coefficients computed directly from the observed data pairs.

We apply this optimization method for affine vector fields to both
reversible and non-reversible OU processes. Figure~\ref{fig:non_reversible_plot_mu_cov_vs_t} compares the time
evolution of selected mean and covariance components of the marginal
distributions obtained from the computed probability flows with
their exact counterparts. The numerical results show good agreement with the exact solutions in both cases.

\begin{figure}[htb!]
\centering
\begin{subfigure}{0.24\textwidth}
    \centering
    \includegraphics[width=0.8\linewidth]{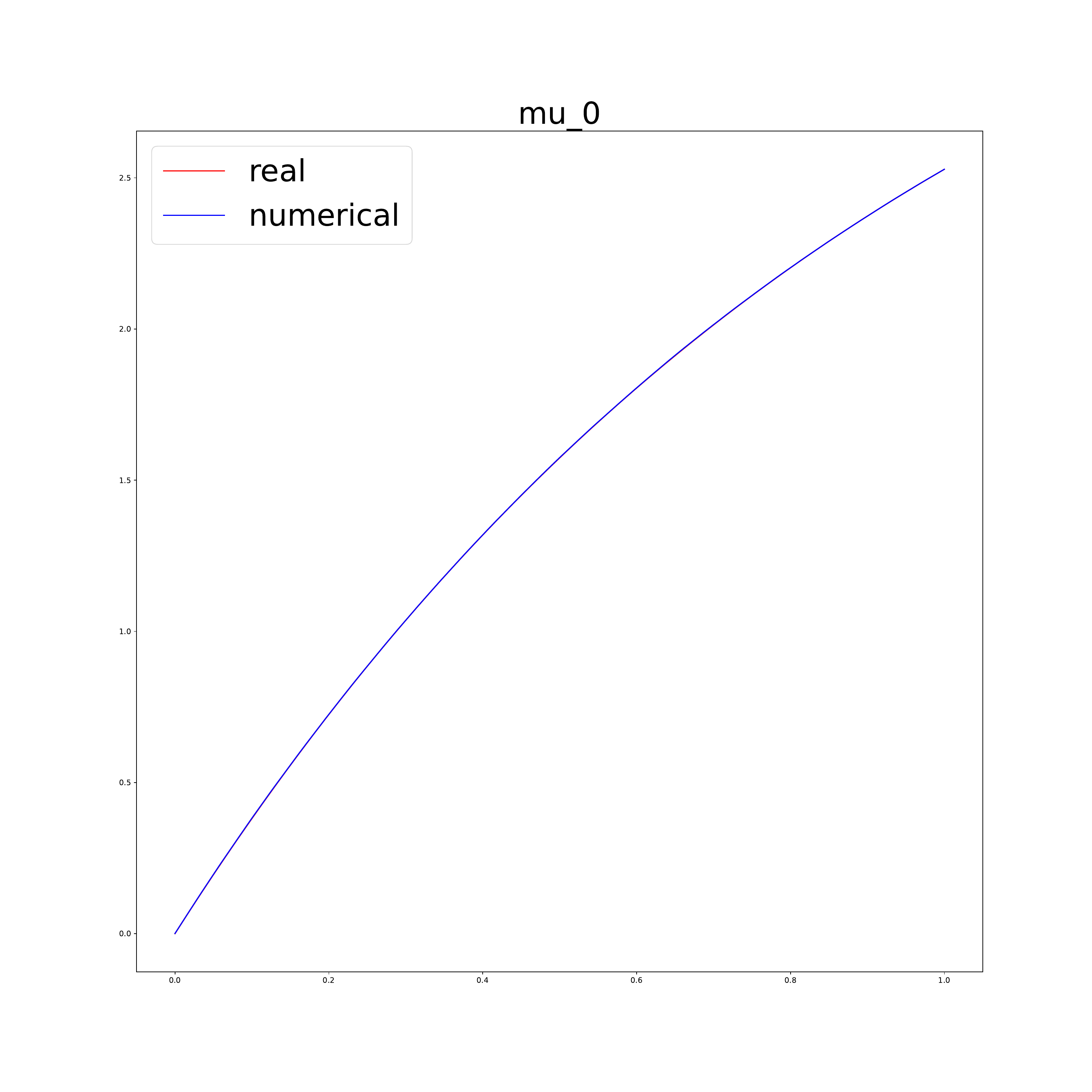}
    \caption{$\mu_{t,0}$ (Reversible)}
\end{subfigure}%
\begin{subfigure}{0.24\textwidth}
    \centering
    \includegraphics[width=0.8\linewidth]{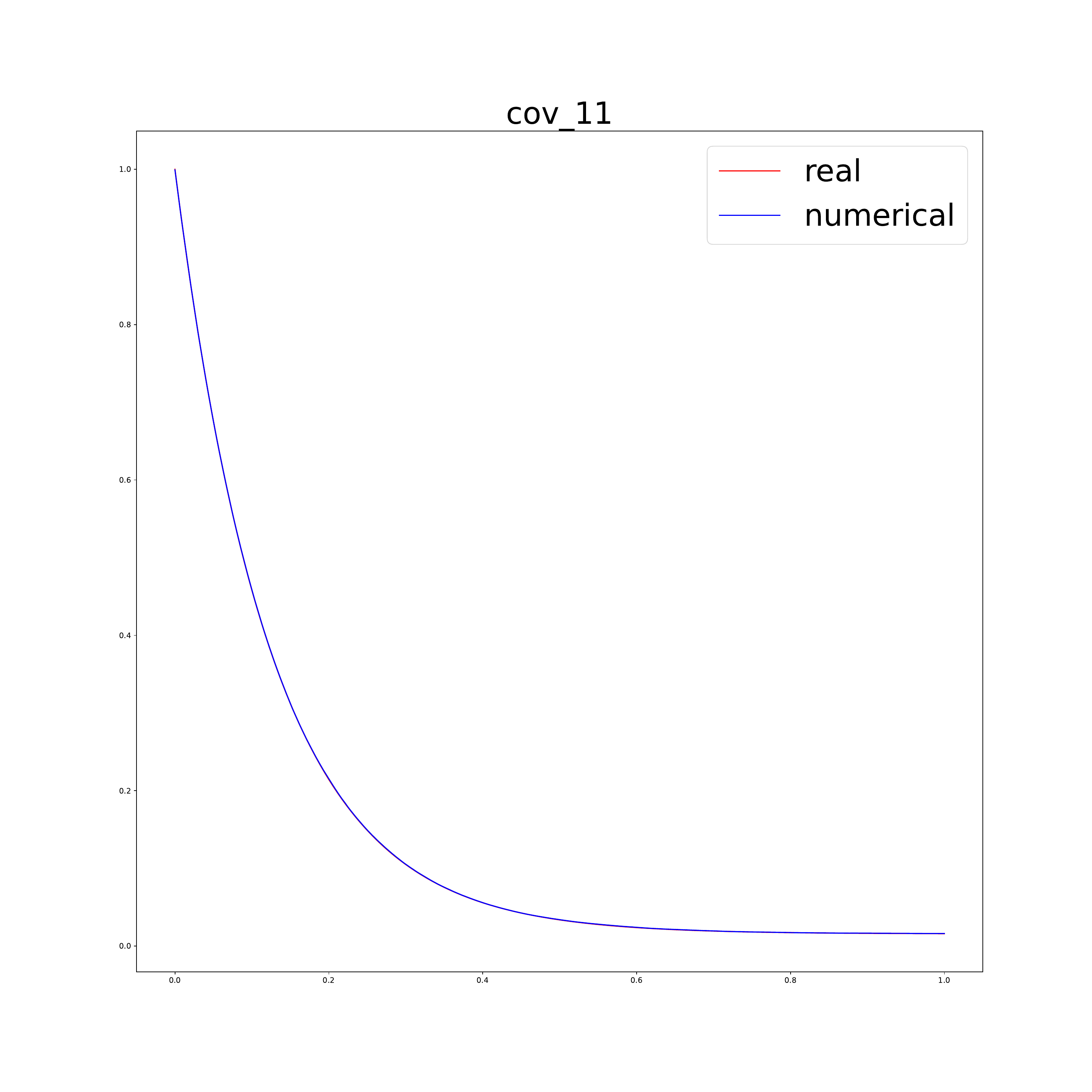}
    \caption{$\boldsymbol{\Sigma}_{t,11}$ (Reversible)}
\end{subfigure}%
\begin{subfigure}{0.24\textwidth}
    \centering
    \includegraphics[width=0.8\linewidth]{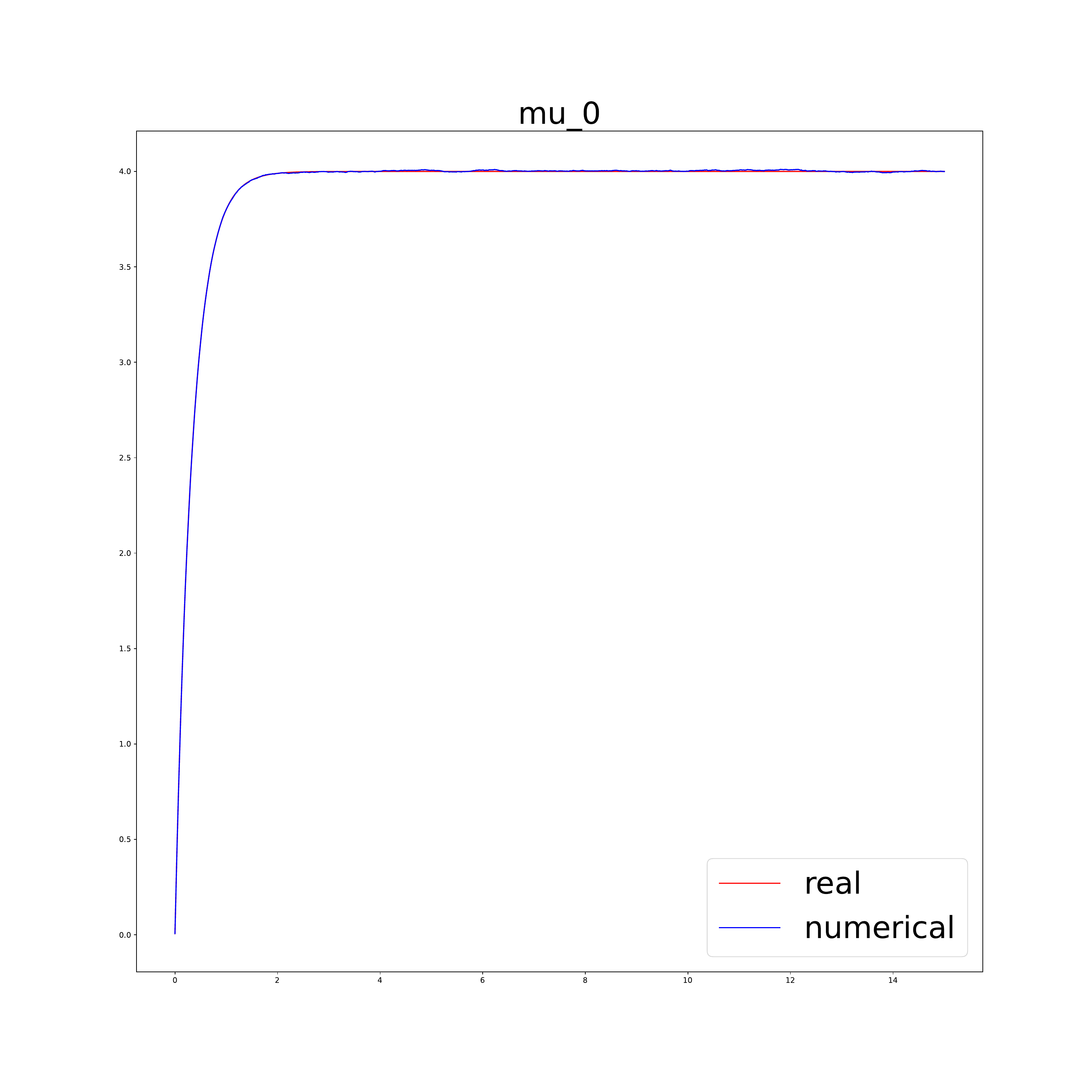}
    \caption{$\mu_{t,0}$ (Non-reversible)}
\end{subfigure}%
\begin{subfigure}{0.24\textwidth}
    \centering
    \includegraphics[width=0.8\linewidth]{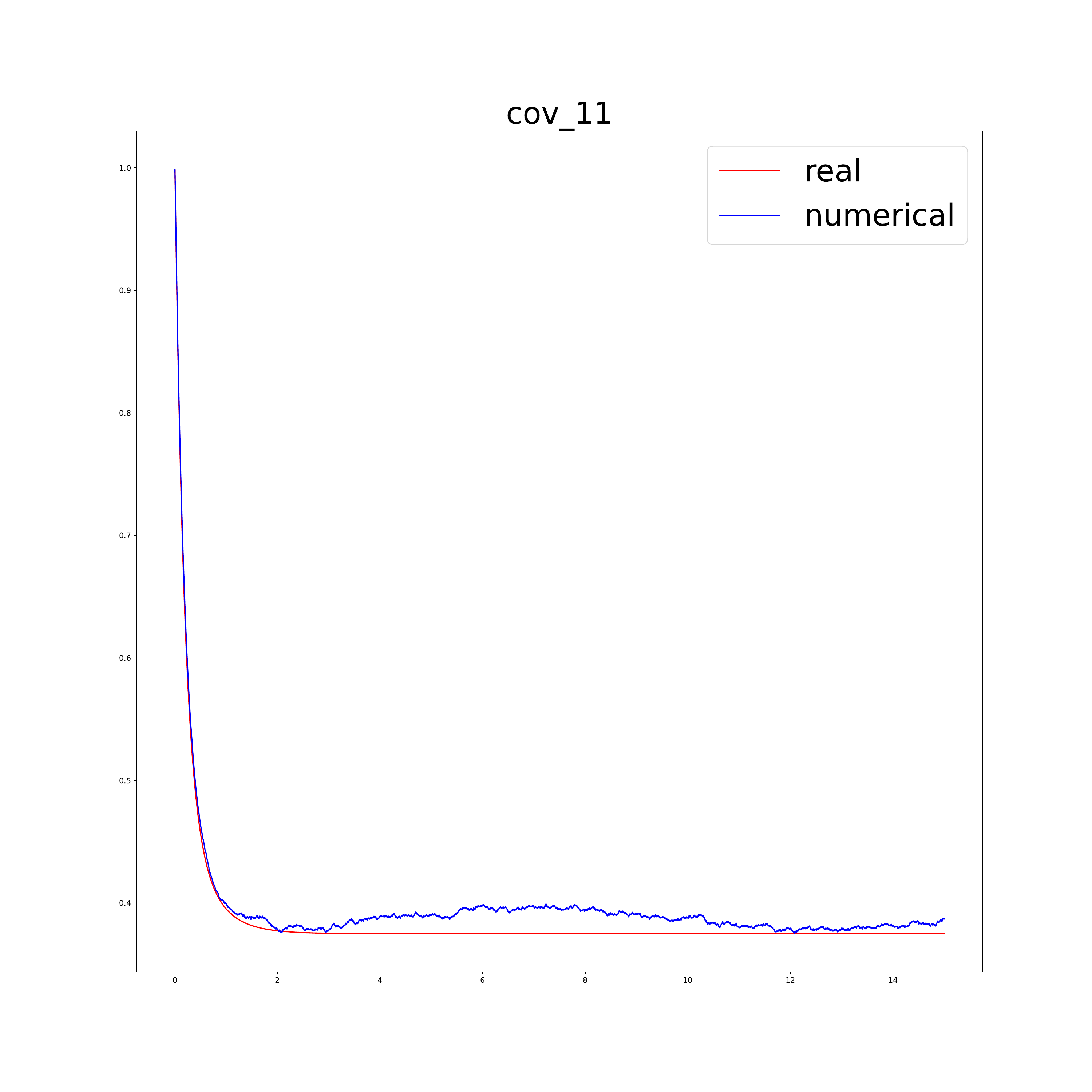}
    \caption{$\boldsymbol{\Sigma}_{t,11}$ (Non-reversible)}
\end{subfigure}%
\caption{
Time evolution of selected mean and covariance components of the
marginal distributions for reversible (left two panels) and
non-reversible (right two panels) OU processes.
Results obtained from the computed probability flows are compared
with the exact solutions.}

\label{fig:non_reversible_plot_mu_cov_vs_t}
\end{figure}

\paragraph{Velocity network and optimization.}
The probability velocity $v_\theta(x,t)$ is represented by an MLP with
five hidden layers and hidden width 50. The network takes the physical
state $x\in\mathbb{R}^2$ and physical time $t$ as input and outputs a
two-dimensional velocity vector. The model is optimized using the
empirical PhiBE-Flow objective described in Section~3.


After training, we generate probability-flow trajectories by numerically
integrating
\[
    \dot X_t=v_\theta(X_t,t).
\]


\paragraph{Velocity-field error.}
Because the exact probability velocity is available, we evaluate the
learned field using
\[
    \left\|
        v_\theta-v^\star
    \right\|_{L^2(\rho\otimes[0,T])},
\]
where the expectation over $\rho_t$ is approximated using independent
samples from the analytical OU marginal. We additionally report the
time-dependent velocity error
\[
    \left\|
        v_\theta(\cdot,t)-v^\star(\cdot,t)
    \right\|_{L^2(\rho_t)}.
\]

\paragraph{Dependence on temporal resolution and sample size.}
To study the effect of temporal resolution, we fix
$N_x=30{,}000$ and vary
\[
    N_t=10\times2^k,
    \qquad
    k=0,1,\ldots,8.
\]
To study finite-sample behavior, we fix $N_t=3{,}000$ and vary
\[
    N_x=100\times2^k,
    \qquad
    k=0,1,\ldots,9.
\]
For each setting, we measure the global
$L^2(\rho\otimes[0,T])$ error between the learned and exact probability
velocities.

\paragraph{Proof of Theorem \ref{thm:ou_gaussian_marginals}.}
\begin{proof}
The explicit solution of \eqref{eq:ou_general_sde} is
\[
    X_t
    =
    \mu+\exp(\Gamma t)(X_0-\mu)
    +\int_0^t\exp(\Gamma(t-s))\sigma\,dW_s.
\]
Since $X_0$ is Gaussian and independent of the Gaussian
stochastic integral, $X_t$ is Gaussian.
Taking expectations gives \eqref{eq:ou_mean_theorem},
while the Itô isometry yields
\begin{equation}
    \Sigma_t
    =
    \exp(\Gamma t)\Sigma_0\exp(\Gamma^\top t)
    +\int_0^t
    \exp(\Gamma s)\sigma\sigma^\top
    \exp(\Gamma^\top s)\,ds.
\end{equation}

Because $\Gamma$ is stable, the integral
\[
    B
    =
    \int_0^\infty
    \exp(\Gamma s)\sigma\sigma^\top
    \exp(\Gamma^\top s)\,ds
\]
converges and defines the unique symmetric positive-definite
solution of \eqref{sylvester_equ}. Moreover,
\[
    \frac{\mathrm{d}}{\mathrm{d}s}
    \left(
        \exp(\Gamma s)B\exp(\Gamma^\top s)
    \right)
    =
    -\exp(\Gamma s)\sigma\sigma^\top\exp(\Gamma^\top s).
\]
Integrating this identity gives
\[
    \int_0^t
    \exp(\Gamma s)\sigma\sigma^\top
    \exp(\Gamma^\top s)\,ds
    =
    B-\exp(\Gamma t)B\exp(\Gamma^\top t).
\]
Substitution into the covariance formula establishes
\eqref{eq:ou_covariance_theorem}.
Finally, $\exp(\Gamma t)\rightarrow0$ as $t\rightarrow\infty$,
which gives the stated limits and the stationary distribution.
\end{proof}

\subsection{Double Pendulum System}
\label{app:db}

We use the standard Acrobot configuration with
$m_1=m_2=1$, $l_1=l_2=1$, $l_{c1}=l_{c2}=0.5$,
$I_1=I_2=1$, and gravitational acceleration $g=9.8$.
Stochasticity is introduced by adding independent Brownian perturbations
to the two angular accelerations.

For the state-space experiment, we generate 1000 trajectories, each
containing 240 observations corresponding to 8 seconds of simulation at
30 frames per second. Initial states are sampled uniformly from
\[
\Omega =
\left\{
(\theta_1,\theta_2,\dot{\theta}_1,\dot{\theta}_2):
-\pi \leq \theta_1,\theta_2 \leq \pi,\;
-0.1 \leq \dot{\theta}_1,\dot{\theta}_2 \leq 0.1
\right\}.
\]
The probability-velocity model is implemented as a residual MLP with
hidden dimension 256 and two residual blocks. We optimize the model
using Adam with learning rate $10^{-3}$ for 50 epochs. Results are
generated autoregressively using forward Euler integration.

\subsubsection{State-Space Implementation Details}
\label{app:db_state_impl}

The state-space model operates directly on
$x_t=(\theta_1,\theta_2,\dot{\theta}_1,\dot{\theta}_2)_t$.
Consecutive observations $(x_t,x_{t+1})$ are used to construct the
empirical drift and diffusion terms required by the training objective.
Algorithm~\ref{alg:state-time-input} summarizes training and autoregressive
generation.

\begin{algorithm}[t]
\caption{State-space training and autoregressive generation with time input}
\label{alg:state-time-input}
\begin{algorithmic}[1]
\Statex \textbf{Part I: Training}
\Require Trajectories $\mathcal D=\{(x_i^{(\ell)},t_i^{(\ell)})_{i=0}^{M-1}\}_{\ell=1}^{N_{\mathrm{traj}}}$ with spacing $\Delta t$; model $v_\omega(x,t)$; learning rate $\eta=10^{-3}$; epochs $E=50$
\For{$e=1,\ldots,E$}
  \For{each consecutive pair $(x_i,x_{i+1})$ and its timestamp $t_i$ in $\mathcal D$}
    \State $\delta x_i\gets x_{i+1}-x_i$
    \State $v\gets v_\omega(x_i,t_i)$
    \State $b\gets\delta x_i/\Delta t$, \quad $A\gets\delta x_i\delta x_i^\top/\Delta t$
    \State $J\gets\left.\nabla_x v_\omega(x,t_i)\right|_{x=x_i}$ \Comment{hold time fixed}
    \State $\mathcal L\gets\|v\|_2^2-2\langle b,v\rangle-\operatorname{Tr}(AJ)$
    \State $\omega\gets\operatorname{Adam}(\omega,\nabla_\omega\mathcal L,\eta)$
  \EndFor
\EndFor
\Statex \textbf{Part II: Autoregressive generation}
\Require Initial state $x_{\mathrm{in}}$ at time $t_{\mathrm{in}}$; trained $v_\omega$; number of forecast steps $H$; time step $\Delta t$
\State $\widehat x_0\gets x_{\mathrm{in}}$, \quad $\tau_0\gets t_{\mathrm{in}}$
\State $\mathcal X_{\mathrm{gen}}\gets[(\widehat x_0,\tau_0)]$
\For{$h=0,\ldots,H-1$}
  \State $v\gets v_\omega(\widehat x_h,\tau_h)$
  \State $\widehat x_{h+1}\gets\widehat x_h+\Delta t\,v$
  \State $\tau_{h+1}\gets\tau_h+\Delta t$
  \State Append $(\widehat x_{h+1},\tau_{h+1})$ to $\mathcal X_{\mathrm{gen}}$
\EndFor
\State \Return $\mathcal X_{\mathrm{gen}}$
\end{algorithmic}
\end{algorithm}

\subsubsection{Image Data Generation}
\label{app:db_image_data}

For the image-based experiment, each physical trajectory is converted
into a sequence of visual observations. Given the joint angles
$(\theta_1,\theta_2)$, the Cartesian coordinates of the first and second
joints are computed as
\[
x_1 = \sin(\theta_1),
\qquad
y_1 = -\cos(\theta_1),
\]
and
\[
x_2 = x_1 + \sin(\theta_1+\theta_2),
\qquad
y_2 = y_1 - \cos(\theta_1+\theta_2).
\]
Both links have unit length. The pendulum is rendered inside the
Cartesian domain $[-2,2]\times[-2,2]$, with rods and joints shown in
white against a black background. Each rendered frame is converted to
grayscale and resized to the input resolution of the encoder.

For representation learning and latent-space generation, we generate
1000 image trajectories, each containing 240 consecutive frames.

\subsubsection{Geometry-Preserving Encoder and Decoder}
\label{app:db_gpe}

We employ a geometry-preserving encoder $T_\phi$~\citep{lee2025geometry}
to map an image $x_t$ to a three-dimensional latent representation
$z_t=T_\phi(x_t)$ and a decoder $S_\psi$ to reconstruct the image from
the latent representation. During decoder training, Gaussian
perturbations are added to the latent variables to improve robustness
to errors accumulated during autoregressive prediction.

The encoder and decoder architectures used for the double-pendulum
experiment are summarized in Table~\ref{tab:db_gpe_arch}.

\begin{table}[H]
\centering
\caption{Architectures of the geometry-preserving encoder
$T_\phi$ and decoder $S_\psi$.}
\label{tab:db_gpe_arch}
\resizebox{\linewidth}{!}{
\begin{tabular}{@{}llc@{}}
\toprule
\textbf{Model} & \textbf{Layer} & \textbf{Output Shape} \\
\midrule

\textbf{Encoder}
& Input image $x$ (resized)
& $1\times32\times32$ \\

& Conv2d($1\rightarrow32$), LeakyReLU, BatchNorm
& $32\times16\times16$ \\

& Conv2d($32\rightarrow64$), LeakyReLU
& $64\times8\times8$ \\

& Conv2d($64\rightarrow128$), LeakyReLU, BatchNorm
& $128\times4\times4$ \\

& Flatten, Linear($2048\rightarrow3$)
& $3$ \\

\midrule

\textbf{Decoder}
& Input latent $z$
& $3\times1\times1$ \\

& ConvTranspose2d($3\rightarrow64$),
  Conv2d($64\rightarrow128$)
& $128\times4\times4$ \\

& ConvTranspose2d($128\rightarrow64$),
  Conv2d($64\rightarrow64$)
& $64\times8\times8$ \\

& ConvTranspose2d($64\rightarrow32$),
  Conv2d($32\rightarrow32$)
& $32\times16\times16$ \\

& ConvTranspose2d($32\rightarrow128$),
  Conv2d($128\rightarrow1$), Tanh
& $1\times32\times32$ \\

\bottomrule
\end{tabular}
}
\end{table}

\subsubsection{Encoder and Decoder Training}
\label{app:db_gpe_train}

The encoder is first trained to preserve pairwise geometry in the latent
space while encouraging temporally adjacent observations to remain
close. After the encoder is fixed, the decoder is trained to reconstruct
images from noisy latent representations. Algorithm~\ref{alg:db_gpe}
summarizes the procedure.

\begin{algorithm}[!ht]
\caption{Training the Geometry-Preserving Encoder and Decoder}
\label{alg:db_gpe}
\begin{algorithmic}[1]

\Require Dataset
$\mathcal{D}=\{X_1^l,X_2^l,\ldots,X_{240}^l\}_{l=1}^{1000}$,
encoder $T_\phi$, decoder $S_\psi$,
$\lambda_{\mathrm{smooth}}=25$,
noise scale $\sigma=0.3$,
learning rate $\eta=10^{-4}$,
batch size $B=100$

\Statex \textbf{Stage I: Encoder Training}

\For{$k=1,2,\ldots,50{,}000$}
    \State Sample
    $\mathcal{B}
    =\{(x_t^{(b)},x_{t+1}^{(b)})\}_{b=1}^{B}$

    \For{$b=1,2,\ldots,B$}
        \State $z_t^{(b)} \leftarrow T_\phi(x_t^{(b)})$
        \State $z_{t+1}^{(b)} \leftarrow T_\phi(x_{t+1}^{(b)})$
    \EndFor

    \State $\displaystyle
    \mathcal{L}_{\mathrm{GME}}^{(t)}
    \leftarrow
    \frac{1}{B^2}
    \sum_{i,j=1}^{B}
    \left[
    \log\!\left(1+\|z_t^{(i)}-z_t^{(j)}\|_2^2\right)
    -
    \log\!\left(1+\|x_t^{(i)}-x_t^{(j)}\|_2^2\right)
    \right]^2$

    \State $\displaystyle
    \mathcal{L}_{\mathrm{GME}}^{(t+1)}
    \leftarrow
    \frac{1}{B^2}
    \sum_{i,j=1}^{B}
    \left[
    \log\!\left(1+\|z_{t+1}^{(i)}-z_{t+1}^{(j)}\|_2^2\right)
    -
    \log\!\left(1+\|x_{t+1}^{(i)}-x_{t+1}^{(j)}\|_2^2\right)
    \right]^2$

    \State $\displaystyle
    \mathcal{L}_{\mathrm{GPE}}
    \leftarrow
    \frac{1}{2}
    \left(
    \mathcal{L}_{\mathrm{GME}}^{(t)}
    +
    \mathcal{L}_{\mathrm{GME}}^{(t+1)}
    \right)$

    \State $\displaystyle
    \mathcal{L}_{\mathrm{smooth}}
    \leftarrow
    \frac{1}{B}
    \sum_{b=1}^{B}
    \|z_{t+1}^{(b)}-z_t^{(b)}\|_2^2$

    \State $\mathcal{L}_{\mathrm{enc}}
    \leftarrow
    \mathcal{L}_{\mathrm{GPE}}
    +
    \lambda_{\mathrm{smooth}}
    \mathcal{L}_{\mathrm{smooth}}$

    \State $\phi \leftarrow
    \operatorname{Adam}
    \bigl(\phi,\nabla_\phi\mathcal{L}_{\mathrm{enc}},\eta\bigr)$
\EndFor

\vspace{0.15cm}
\Statex \textbf{Stage II: Decoder Training}

\For{$k=1,2,\ldots,100{,}000$}
    \State Sample
    $\mathcal{B}
    =\{(x_t^{(b)},x_{t+1}^{(b)})\}_{b=1}^{B}$

    \For{$b=1,2,\ldots,B$}
        \State $z_t^{(b)} \leftarrow T_\phi(x_t^{(b)})$
        \State $z_{t+1}^{(b)} \leftarrow T_\phi(x_{t+1}^{(b)})$

        \State $\epsilon_t^{(b)},\epsilon_{t+1}^{(b)}
        \sim\mathcal{N}(0,\sigma^2 I)$

        \State $\hat{x}_t^{(b)}
        \leftarrow
        S_\psi(z_t^{(b)}+\epsilon_t^{(b)})$

        \State $\hat{x}_{t+1}^{(b)}
        \leftarrow
        S_\psi(z_{t+1}^{(b)}+\epsilon_{t+1}^{(b)})$
    \EndFor

    \State $\displaystyle
    \mathcal{L}_{\mathrm{recon}}
    \leftarrow
    \frac{1}{2B}
    \sum_{b=1}^{B}
    \left(
    \|x_t^{(b)}-\hat{x}_t^{(b)}\|_2^2
    +
    \|x_{t+1}^{(b)}-\hat{x}_{t+1}^{(b)}\|_2^2
    \right)$

    \State $\psi \leftarrow
    \operatorname{Adam}
    \bigl(\psi,\nabla_\psi\mathcal{L}_{\mathrm{recon}},\eta\bigr)$
\EndFor

\end{algorithmic}
\end{algorithm}

\subsubsection{Latent-Space Dynamics}
\label{app:db_latent}

Once the image representation is learned, each encoded latent position
$z_t\in\mathbb{R}^3$ is augmented with its finite-difference velocity,
\[
\dot{z}_t
=
\frac{z_t-z_{t-1}}{\Delta t},
\]
yielding the six-dimensional dynamical state
\[
s_t =
\begin{bmatrix}
z_t \\
\dot{z}_t
\end{bmatrix}
\in\mathbb{R}^6.
\]

We use a neural network $v_\omega$ to learn the probability velocity in
this latent dynamical space. The specific neural architecture can be
chosen according to the dataset and is not essential to the proposed
training objective.

\subsubsection{Latent-Space Training and Prediction}
\label{app:db_latent_impl}

Algorithm~\ref{alg:latent-time-input} describes training and autoregressive
prediction in the learned latent space. The same probability-velocity
objective as in the explicit state-space experiment is applied to the
latent dynamical states.

\begin{algorithm}[t]
\caption{Latent-space training and single-state autoregressive prediction with time input}
\label{alg:latent-time-input}
\begin{algorithmic}[1]
\Statex \textbf{Part I: Training}
\Require Latent trajectories 
$\{(x_i^{(\ell)},t_i^{(\ell)})_{i=0}^{M-1}\}_{\ell=1}^{N_{\mathrm{traj}}}$;
model $v_\omega(x,t)$; learning rate $\eta=10^{-4}$;
time step $\Delta t$; epochs $E=50$

\For{$e=1,\ldots,E$}
  \For{each consecutive pair $(x_i,x_{i+1})$}
    \State $\delta x_i\gets x_{i+1}-x_i$
    \State $v\gets v_\omega(x_i,t_i)$
    \State $b\gets\delta x_i/\Delta t$,
    \quad $A\gets\delta x_i\delta x_i^\top/\Delta t$
    \State $J\gets\left.\nabla_x v_\omega(x,t_i)\right|_{x=x_i}$
    \State $\mathcal L\gets
    \|v\|_2^2-2\langle b,v\rangle-\operatorname{Tr}(AJ)$
    \State $\omega\gets
    \operatorname{Adam}(\omega,\nabla_\omega\mathcal L,\eta)$
  \EndFor
\EndFor

\Statex \textbf{Part II: Single-state autoregressive prediction}
\Require Initial latent state $\widehat x_0$ at time $t_{\mathrm{in}}$;
trained $v_\omega$; decoder $S_\psi$;
number of forecast steps $H$; time step $\Delta t$

\State $\tau_0\gets t_{\mathrm{in}}$,
\quad $\mathcal X_{\mathrm{pred}}\gets[\ ]$

\For{$h=0,\ldots,H-1$}
  \State $v\gets v_\omega(\widehat x_h,\tau_h)$
  \State $\widehat x_{h+1}
  \gets \widehat x_h+\Delta t\,v$
  \State $\tau_{h+1}\gets\tau_h+\Delta t$
  \State $\widehat y_{h+1}
  \gets S_\psi(\widehat x_{h+1})$
  \State Append $(\widehat y_{h+1},\tau_{h+1})$
  to $\mathcal X_{\mathrm{pred}}$
\EndFor

\State \Return $\mathcal X_{\mathrm{pred}}$
\end{algorithmic}
\end{algorithm}

\subsection{2D Navier--Stokes Equations Generation}
\label{app:nse}

\paragraph{Stochastic Navier--Stokes dataset.}We use the two-dimensional stochastic Navier--Stokes dataset from \citet{chen2024probabilistic}, consisting of vorticity fields on the periodic domain $[0,2\pi]^2$. The dynamics incorporate viscosity $\nu=10^{-3}$, linear damping $\alpha=0.1$, and white-in-time stochastic forcing with amplitude $\varepsilon=1$ acting on selected Fourier modes. The data were generated using a pseudo-spectral solver with Euler--Maruyama time stepping on a $256\times256$ grid, with integration step $\delta t=10^{-4}$. The dataset comprises 1,000 trajectories simulated over $t\in[0,100]$, with the initial transient $t\in[0,50]$ discarded. Vorticity snapshots were recorded every $\Delta t=0.5$, yielding $2\times10^5$ snapshots, and subsequently downsampled to $128\times128$ resolution.

\subsection{KTH Video Generation}
\label{app:kth}

\paragraph{Pretrained video autoencoder.}
Following \citet{chen2024probabilistic}, we use a pretrained VQGAN~\citep{esser2021taming} to represent KTH videos in a compact latent space. We download the KTH autoencoder checkpoint released by RIVER~\citep{davtyan2023efficient}, which is also used in their experiments. The encoder maps each $64\times64$ frame to a $4\times8\times8$ latent representation. PhiBE-Flow learns the probability velocity field in this latent space, and the generated latent trajectories are mapped back to video frames using the pretrained decoder.

\paragraph{Evaluation metrics.}
We evaluate video results using Mean Squared Error (MSE),
Peak Signal-to-Noise Ratio (PSNR), Structural Similarity Index Measure
(SSIM), and Fr\'echet Video Distance (FVD). MSE, PSNR, and SSIM evaluate
frame-level reconstruction quality, while FVD measures the discrepancy
between real and generated videos in a learned spatiotemporal feature
space.

\paragraph{MSE.}
The mean squared error is
\begin{equation}
\mathrm{MSE}
=
\frac{1}{N}
\sum_{i=1}^{N}
\left(
\hat{x}_i-x_i
\right)^2,
\end{equation}
where $x_i$ and $\hat{x}_i$ denote the ground-truth and predicted pixel
values, respectively.

\paragraph{PSNR.}
Peak signal-to-noise ratio is computed as
\begin{equation}
\mathrm{PSNR}
=
10\log_{10}
\left(
\frac{\mathrm{MAX}_I^2}{\mathrm{MSE}}
\right),
\end{equation}
where $\mathrm{MAX}_I$ denotes the maximum possible pixel value.

\paragraph{SSIM.}
Structural similarity is computed as
\begin{equation}
\mathrm{SSIM}(x,\hat{x})
=
\frac{
(2\mu_x\mu_{\hat{x}}+C_1)
(2\sigma_{x\hat{x}}+C_2)
}{
(\mu_x^2+\mu_{\hat{x}}^2+C_1)
(\sigma_x^2+\sigma_{\hat{x}}^2+C_2)
},
\end{equation}
where $\mu$, $\sigma^2$, and $\sigma_{x\hat{x}}$ denote the mean,
variance, and covariance, respectively, and $C_1$ and $C_2$ are
constants for numerical stability.

\paragraph{FVD.}
Fr\'echet Video Distance compares the distributions of real and
generated videos in a spatiotemporal feature space:
\begin{equation}
\mathrm{FVD}
=
\left\|
\boldsymbol{\mu}_r-\boldsymbol{\mu}_g
\right\|_2^2
+
\mathrm{Tr}
\left(
\boldsymbol{\Sigma}_r
+
\boldsymbol{\Sigma}_g
-
2
\left(
\boldsymbol{\Sigma}_r
\boldsymbol{\Sigma}_g
\right)^{1/2}
\right),
\end{equation}
where $(\boldsymbol{\mu}_r,\boldsymbol{\Sigma}_r)$ and
$(\boldsymbol{\mu}_g,\boldsymbol{\Sigma}_g)$ denote the mean and
covariance of the corresponding video feature distributions.

We follow the public Google Research FVD evaluation protocol using an
I3D feature extractor pretrained on Kinetics-400. The same pretrained
feature extractor and evaluation implementation are used for all
compared methods.

The FVD implementation is taken from the public Google Research
evaluation code\footnote{\url{https://github.com/google-research/google-research/tree/master/frechet_video_distance}},
with an I3D feature extractor pretrained on
Kinetics-400\footnote{\url{https://github.com/google-deepmind/kinetics-i3d}}.
The corresponding pretrained I3D weights are obtained from the public
evaluation code\footnote{\url{https://onedrive.live.com/download?cid=78EEF3EB6AE7DBCB\&resid=78EEF3EB6AE7DBCB\%21199\&authkey=AApKdFHPXzWLNyI}}
and are kept fixed for all methods.

\end{document}